\documentclass{article}

\usepackage{microtype}
\usepackage{graphicx}
\usepackage{subcaption}
\usepackage{booktabs} 

\usepackage{hyperref}

\usepackage{xurl}

\usepackage[accepted]{icml2026}

\usepackage{amsmath}
\usepackage{amssymb}
\usepackage{mathtools}
\usepackage{amsthm}

\usepackage[capitalize,noabbrev]{cleveref}

\theoremstyle{plain}
\newtheorem{theorem}{Theorem}[section]

\usepackage{aliascnt}
\newaliascnt{lemma}{theorem}
\newtheorem{lemma}[lemma]{Lemma}
\aliascntresetthe{lemma}

\theoremstyle{definition}

\theoremstyle{remark}

\usepackage[textsize=tiny]{todonotes}

\usepackage{csquotes}
\usepackage[acronym]{glossaries}
\usepackage{glossaries-extra}
\usepackage{algorithm}

\usepackage{amsthm, bm, bbm, upgreek}
\usepackage{siunitx, commath, nicefrac}
\DeclareSIUnit{\nounit}{\relax}
\usepackage[short]{optidef}

\usepackage{url}
\usepackage[linesnumbered,ruled,vlined,noend,algo2e]{algorithm2e}
\usepackage{nicefrac}
\usepackage{multirow, multicol, makecell}

\newcommand{\nR}{\mathbb{R}}

\newcommand{\xmin}{{x_\text{min}}}
\newcommand{\xmax}{{x_\text{max}}}
\newcommand{\vmax}{{v_\text{max}}}
\newcommand{\amax}{{a_\text{max}}}
\newcommand{\tmax}{{t_\text{max}}}
\newcommand{\Cmin}{{C_\text{min}}}
\newcommand{\Cmax}{{C_\text{max}}}
\newcommand{\kmax}{{k_\text{max}}}
\newcommand{\vwall}{{v_\text{wall}}}
\newcommand{\piana}{{\pi_\text{ana}}}
\newcommand{\pioptxzero}{\pi_{\text{opt},x_0}}

\DeclareMathOperator*{\sign}{sign}
\DeclareMathOperator*{\std}{std}

\setabbreviationstyle[acronym]{long-short}
\newacronym{ai}{AI}{Artificial Intelligence}
\newacronym{a2c}{A2C}{Advantage Actor-Critic}
\newacronym{ars}{ARS}{Augmented Random Search}
\newacronym{api}{API}{Application Programming Interface}
\newacronym{drl}{DRL}{Deep Reinforcement Learning}
\newacronym{ddpg}{DDPG}{Deep Deterministic Policy Gradient}
\newacronym{ml}{ML}{Machine Learning}
\newacronym{mlp}{MLP}{Multilayer Perceptron}
\newacronym{rl}{RL}{Reinforcement Learning}
\newacronym{td3}{td3}{Twin Delayed Deep Deterministic Policy Gradient}
\newacronym{trpo}{TRPO}{Trust Region Policy Optimization}
\newacronym{tqc}{TQC}{Truncated Quantile Critics}
\newacronym{ode}{ODE}{ordinary differential equation}
\newacronym{ppo}{PPO}{Proximal Policy Optimization}
\newacronym{sota}{SOTA}{state of the art}
\newacronym{mdp}{MDP}{Markov Decision Process}
\newacronym{sl}{SL}{supervised learning}
\newacronym{sac}{SAC}{Soft Actor-Critic}
\newacronym{sb3}{SB3}{Stable Baselines3}

\icmltitlerunning{Chebyshev Policies and the Mountain Car Problem}

\begin{document}

\twocolumn[
  \icmltitle{Chebyshev Policies and the Mountain Car Problem:\\
    Reinforcement Learning for Low-Dimensional Control Tasks}



  \icmlsetsymbol{equal}{*}

  \begin{icmlauthorlist}
    \icmlauthor{Stefan Huber}{fhs}
    \icmlauthor{Hannes Unger}{fhs}
    \icmlauthor{Georg Schäfer}{fhs}
    \icmlauthor{Jakob Rehrl}{fhs}
  \end{icmlauthorlist}

  \icmlaffiliation{fhs}{Josef Ressel Centre for Intelligent and Secure Industrial Automation,
  University of Applied Sciences, Salzburg, Austria}

  \icmlcorrespondingauthor{Stefan Huber}{stefan.huber@fh-salzburg.ac.at}

  \icmlkeywords{Chebyshev Polynomials, Mountain Car Problem, Low-dimensional Control, Optimal Control, Reinforcement Learning}

  \vskip 0.3in
]



\printAffiliationsAndNotice{}  

\glsunset{ode}

\newcommand{\codeurl}{\cite{UngerSchaefer2026}}

\begin{abstract}
  We analytically solve the Mountain Car problem, a canonical benchmark in RL,
  and derive an optimal control solution, closing a gap after 36 years.
  This enables us to reveal two surprising insights: The optimal control is
  quite simple, yet modern RL agents display a large gap to optimality.
  Motivated by the analysis of the optimal control, we introduce Chebyshev
  policies as a universal (i.e.\ dense) class of RL policies from first
  principles.
  They can be trained as drop-in replacements of neural nets, reducing the
  regret by a factor of \num{6.18}, while requiring 277 times fewer parameters,
  fostering sample efficiency, explainability and real-time capability.
  Chebyshev policies are evaluated on further RL tasks, including a real-world
  non-linear motion control testbed. They consistently improve performance over
  neural nets with PPO, ARS and REINFORCE.
  Our results demonstrate how Chebyshev policies offer a compelling and
  lightweight alternative or addition to neural nets for low-dimensional
  control tasks.
\end{abstract}

\section{Introduction}

\subsection{Motivation}

\gls{rl} underwent a remarkable progress and became a very powerful paradigm to
tackle a large variety of control and decision-making tasks \cite{Sutton2018}
with plenty of applications in virtual and real-world environments.
At the same time, \gls{rl} also faces a number of challenges, especially when
applied to real-world tasks. Dulac-Arnold et al.~\cite{DulacArnold2021}
identified nine such challenges, including sample efficiency, explainability
and interpretability, real-time capability and training stability, see also
more recent surveys by Tang et al.~\cite{Tang2025} and Gazi et
al.~\cite{Gazi2026}. We also lack understanding on theoretical foundations, for
instance on \gls{rl} training dynamics and implicit regularization
\cite{Eysenbach2023}.

\begin{figure}[tb]
    \centering
    \includegraphics[width=0.85\linewidth]{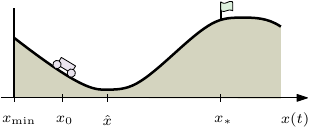}
    \caption{The car starts at $x_0$ and has to reach the goal at $x_*$ against
  gravity. There is an inelastic wall at $\xmin$.}
    \label{fig:mountaincar}
\end{figure}

In this paper, we take a step back and address these fundamental challenges by
revisiting a classic \gls{rl} benchmark task from ground up, namely the
\emph{Mountain Car} problem \cite{Moore1990}, and then consequently draw
conclusions from our learnings. As \cref{fig:mountaincar} illustrates, a car
shall apply minimum total propulsion to overcome gravity and reach the goal at
the top. The propulsion is too small to reach the goal at once, illustrating
the concept of delayed reward and exploration.

The Mountain Car problem is interesting, because the underlying system dynamics
is of a typical form for physics and engineering problems, while still being
simple enough to hope for an analytic treatment.
Still, the optimal solution to this problem is unknown for 36 years.
Consequently, the \emph{regret} (gap to optimality) of the current \gls{sota}
is unknown, too, leaving a simple, yet central question wide open: \emph{How
  close to optimality are current \gls{rl} algorithms actually on this
canonical benchmark task?}

\paragraph{Contribution}

We analytically derive an optimal control solution, which allows us to evaluate
the regret of the currently leading \gls{rl} agents, revealing two surprises:
The optimal solution is surprisingly simple and the currently top-performing
\gls{rl} agents display a surprisingly high regret.

Motivated by these findings, we facilitate a multi-variate generalization of
Chebyshev polynomials for a novel model of stochastic policies in \gls{rl},
which we call \emph{Chebyshev policies}. They form universal approximators in
the sense that they yield a dense subset of the space of continuous policies.
We evaluate them with \gls{ppo}, \gls{ars} and REINFORCE on the Mountain Car problem and
reduce the regret by a factor of \num{6.18}, while having a factor of \num{277}
less trainable parameters. This naturally addresses sample efficiency,
explainability, interpretability and real-time capabilities, and therefore
addresses five of nine key challenges of real-world \gls{rl} identified by
Dulac-Arnold \cite{DulacArnold2021}.

We evaluate Chebyshev policies on further low-dimensional control tasks, namely
the Pendulum environment of Gymnasium and the real-world helicopter-like Aero~2
testbed with non-linear dynamics. On all tested environments, Chebyshev
policies clearly improve upon \gls{mlp}-based policies, for both \gls{ppo} and \gls{ars},
and they improve on the sim-to-real transfer and the control dynamics.
Furthermore, our analysis of Mountain Car suggests two richer variants of this
benchmark that mitigates some simplicities of the optimal control.
All implementation, training and evaluation code is published at \codeurl.

\subsection{Prior Work and State of the Art}

Since the original introduction of the Mountain Car problem
\cite{Moore1990}, it evolved to different versions, e.g., concerning the
landscape function, parameters defining the motion law, or the goal. To the
best of our knowledge, the optimal solution to any known version of the
(continuous) Mountain Car problem is unknown. In particular no analytical
solution has been published so far.

We are considering the continuous Mountain Car problem as defined in
\cite{Gym-mountaincarcontinuous}.
A reward of \num{100} is given upon reaching the goal, forming a trivial upper
bound for the optimal \emph{return} (cumulative reward), which is reduced by
the propulsion applied.
The Gymnasium leaderboard page \cite{Gym-leaderbord} considers an agent to have
solved the problem when an average return of \num{90} for \num{100} consecutive
trials is reached.
RL Baselines3 Zoo is an \gls{rl} framework for the popular \gls{rl} library
Stable-Baselines3 compatible to Gymnasium environments. The framework offers
loading pre-trained agents, including for the continuous Mountain Car problem.
Mean returns of benchmark results \cite{Raffin2024} are in the range of
\numrange{91}{97}. The current \gls{sota} agents are trained by \gls{ars}
\cite{Mania2018}, \gls{sac} \cite{Haarnoja2018} and \gls{ppo} \cite{Schulman2017} with
expected returns \num{96.77}, \num{94.66}, and \num{94.22}, respectively, as
reconfirmed by our own evaluation in \cref{sec:evaluation}.

These agents facilitate \glspl{mlp} as approximators, typically with two hidden
layers. 
Also models with less parameters and non-parametric models, like Gaussian
Processes, have been explored for the use in \gls{rl} \cite{Grande2014}.
\cite{Rajeswaran2017} demonstrated the applicability of linear policies for
continuous control tasks, but not the Mountain Car problem. \cite{Schulman2015}
propose to train a single-layer \gls{mlp} on a number of random Fourier features
$f(s) = \sin(\langle s, v\rangle + \varphi)$ for random vectors $v$ and phase
shifts $\varphi$.  However, it is unclear whether these features would be
universal in the sense of forming a dense subspace of the policy space or how
well they would sample it.

In a \gls{sl} setting, \cite{Waclawek2024} demonstrated convergence
capabilities of Chebyshev polynomials for piecewise $\mathcal{C}^k$-continuous
approximation tasks in autodiff frameworks like PyTorch. However, we are not
aware that multi-variate Chebyshev polynomials have been facilitated as
parametrized models in \gls{sl} or \gls{rl} so far.

\section{Analytical Solution to Mountain Car}
\label{sec:analyticalsolution}

Let us recall \cref{fig:mountaincar}: A car starts at $x_0$ in the vicinity of
the minimum $\hat{x}$ and has to move to a goal $x_*$ close to the maximum,
while spending a minimum amount of propulsion effort against the gravitational
potential.
The constants of the motion laws are such that a single monotone stroke is
insufficient to reach $x_*$, but the car needs to oscillate.

The problem as coded in \cite{Gym-mountaincarcontinuous} can be formulated as
follows: Given a sequence $(\alpha_t)$ of actions $\alpha_t \in [-1, 1]$, the
trajectory $(x_t)$ of the car is governed by the system
\begin{align*}
    \begin{split}
        x_{t+1} - x_t &= v_{t+1}, \\
        v_{t+1} - v_t &= \amax \cdot \alpha_t - g \cdot \cos(3 x_t),
    \end{split}
\end{align*}
with $\amax = 0.0015$, $g = 0.0025$, and $x_t$ and $v_t$ are limited to
$[\xmin, \xmax] = [-1.2, 0.6]$ and $[-\vmax, \vmax] = [-0.07, 0.07]$,
respectively. The initial $x_0$ is taken from the uniform distribution
$\mathcal{U}([-0.6, -0.4])$ and $v_0 = 0$. The pairs $(x_t, v_t)$ form the
states and $\alpha_t \in [-1, 1]$ the actions for the agent.
The objective is to maximize the \emph{return}
\begin{align*}
  R = - 0.1 \cdot \sum_{t \ge 0} \alpha_t^2 + (100 \text{ once goal is reached})
\end{align*}
over all policies $(x_t, v_t) \stackrel{\pi}{\mapsto} \alpha_t$,
where the \enquote{goal is reached} if for some $t_* \le \tmax =  999$ we have
$x_{t_*} \ge x_*$ and $v_{t_*} \ge v_*$ with $(x_*, v_*) = (0.45, 0)$. Given
that $R \ge 0$ is achievable, we can rephrase the problem as follows:
\begin{mini}
  {\pi}{\ell}{\label{eq:contproblem}}{}
  \addConstraint{\forall t\colon}{x_t \in [\xmin, \xmax], v_t \in [-\vmax, \vmax]}
  \addConstraint{\exists {t_*}\colon}{x_{t_*} \ge x_* \wedge v_{t_*} \ge v_*}
\end{mini}
with $\ell = \sum_{t = 0}^{{t_*}} \alpha_t^2$, $v_0 = 0$ and $x_0 \sim
\mathcal{U}([-0.6, -0.4])$. We will denote by $t_*$ the time when the goal is
first reached.

We first translate the problem into a continuous setting as follows. For
$\alpha \colon [0, \infty) \to \mathbb{R}$ solve
\begin{mini}
    {\alpha}{\ell}{}{\label{eq:contopt}}
    \addConstraint{\exists t_* \in [0, \tmax]: \; x(t_*) \ge x_* \wedge \dot x(t_*) \ge v_*}{}{}
    \addConstraint{\alpha(t) \in [-1, 1]}{}{}
    \addConstraint{x(t) \in [\xmin, \xmax], \dot x(t) \in [-\vmax, \vmax]}{}{}
\end{mini}
where
\begin{align}
    \label{eq:contloss}
    \ell = \int_0^{t_*} \alpha(t)^2 \dif t
\end{align}
and $x \colon [0, \infty) \to \nR$ is governed by the nonlinear \gls{ode}
\begin{align}
    \label{eq:ode}
    \begin{split}
      \ddot x &= \amax \cdot \alpha - g \cdot \cos(3 x) \\
              &\quad \text{with} \quad x(0) = x_0, \; \dot x(0) = 0.
    \end{split}
\end{align}

This optimization problem will be solved in three steps: (i) we investigate the
dynamics of \eqref{eq:ode}, (ii) solve the unconstrained loss minimization of
\eqref{eq:contloss} and (iii) reestablish the constraints given in
\eqref{eq:contopt}.

\paragraph{Step 1: Spatial Formulation of the Dynamics}

In order to solve \eqref{eq:ode}, a central idea is to bring it into the form
$\ddot x = -U'(x)$. This requires the right-hand side of \eqref{eq:ode} to be
determined solely by the locus $x$ and to eliminate time: we need be able to
explain the action $\alpha(t)$ as a spatial \emph{action field} $\tilde
\alpha(x)$ instead. This form can then be approached with classic tools from
calculus, including \cref{thm:osc}, which can be found in textbooks like
\cite{Koenigsberger2000}, p.~277, and \cite{HandFinch1998}, p.~125, by the
following consideration: Defining $E = \frac{1}{2} \dot x^2 + U(x)$, observe
that
\begin{align*}
    \od{E}{t}
    = \dot x (\ddot x + U'(x)) = 0,
\end{align*}
saying that $E$ is constant.\footnote{$E$ is called \enquote{energy} and $U$ is
called \enquote{potential}.} Then the following theorem is known for nonlinear
oscillators:
\begin{theorem}[\cite{Koenigsberger2000}, p.~277]
    \label{thm:osc}
    Consider an interval $[e, f]$ with $U(e) = U(f) = E$ and $U'(e) \neq 0 \neq
    U'(f)$ and $U(\xi) < E$ for $\xi \in (e, f)$. Then a solution function
    $\phi$ of $\ddot x = -U'(x)$ periodically oscillates between $e$ and $f$
    with a period
    \begin{align*}
        T = 2 \int_e^f \frac{1}{\sqrt{2(E-U(\xi))}} \dif \xi
    \end{align*}
    and on the interval $[0, \frac{T}{2}]$ the inverse $\phi^{-1} \colon [e, f]
    \to [0, \frac{T}{2}]$ is  given by
    \begin{align*}
        \phi^{-1}(x) = \int_e^x \frac{1}{\sqrt{2(E-U(\xi))}} \dif \xi.
    \end{align*}
\end{theorem}

Assume for a moment we place no action, i.e., $\alpha(t) = 0$ and $U'(x) = -g
\cos(3x)$. Then \cref{thm:osc} says that the car will oscillate forth and back
periodically, $\dot x(e) = \dot x(f) = 0$, and in particular $e = x_0$ of our
optimization problem.\footnote{Note that $\ddot x = -g \cos(3x)$ describes the
  mathematical pendulum; its solution is the incomplete elliptic integral of
the first kind. A suitable physical interpretation of the Mountain Car problem
would be a pendulum where $x$ is the angle and $\alpha(t)$ is additional torque
to be applied to reach the goal.}
Since $U(x_*) > U(x_0)$ for our given constants, the car is not reaching the
goal without action: We have to apply action to lower the potential $U(x_*)$ at
the goal $x_*$.

If we apply action in the direction of $\dot x$, i.e., $\alpha(t) \cdot \dot x
(t) \ge 0$, then the potential $U$ is further lowered, and otherwise it is
raised. We can therefore restrict our consideration to actions with $\alpha
\cdot \dot x \ge 0$. The oscillation period of a pendulum in a potential $U$
only increases by the actions increasing the kinetic energy, which is
summarized as

\begin{lemma}
    \label{lem:oscperiod}
    Denote $T$ as in \cref{thm:osc} when no action is applied. For any $\alpha$
    with $\alpha \cdot \dot x \ge 0$, the roots of $\dot x(t)$ are at least
    $\frac{T}{2}$ apart.
\end{lemma}

Assume the car reaches the goal in finite time $t_*$ then we have finitely many
roots $t_0, \dots, t_{k-1}$ of $\dot x(t)$ before $t_*$ by
\cref{lem:oscperiod}. Note that $t_0 = 0$ and for simplicity we define $t_k
= t_*$. On each interval $[t_{i-1}, t_i]$, which we call a \emph{stroke}, we
have a strictly monotone $x$. We denote $x_i = x(t_i)$, yielding $x_k = x_*$.

Let us focus on each single stroke $I_i = [t_{i-1}, t_i]$ and denote by
$\phi_i$ a solution as in \cref{thm:osc}, which can therefore be inverted on
$I_i$, i.e., $\phi_i^{-1}$ translates $x$ into $t$. This allows us to define
the action field $\tilde \alpha_i$ over $x$ for the $i$-th stroke as
%
    $\tilde \alpha_i = \alpha \circ \phi_i^{-1}$.
%
This notion of an action field now allows us to bring \eqref{eq:ode} into a
purely spatial form $\ddot x = -U_i'(x)$,
where $U_i(x) = U_{i,a}(x) + U_g(x)$ with
\begin{align*}
    U_{i,a}(x) &= -\amax \cdot \int_{x_{i-1}}^{x} |\tilde \alpha_i(z)| \dif z + U_{i-1,a}(x_{i-1}) \\
    \quad\text{and}\quad
    U_g(x) &= \frac{g}{3} \left(\sin(3x) - \sin(3x_0) \right),
\end{align*}
where we define $U_{0,a}(x_0) = 0$. Note that on the $i$-th stroke the domain
of $U_i$ is $[x_{i-1}, x_i]$. Also note that by choice of the integration
constants, we have $U_1(x_0) = U_g(x_0) = U_{1,a}(x_0) = 0$ at the first stroke
$I_1$. Also the $U_i$ of consecutive strokes meet continuously, i.e, $U_i(x_i)
= U_{i+1}(x_i)$. Hence, we have $E = 0$ over all strokes. Consequently, at the
last stroke $I_k$ we require $U_k(x_*) \le -\frac{1}{2} v_*^2$ at the goal,
since $E = \frac{1}{2} \dot x^2 + U_k(x) = 0$.

\begin{figure}[htb]
    \tikzstyle{vertex}=[circle,fill=black,minimum size=4pt,inner sep=0pt]
    \centering

    \def\xgoal{0.45}
    \def\xzero{-0.9}
    \def\xone{0.02}
    \def\xtwo{-1.255}
    \def\xthree{0.305}
    \def\xizero{\xzero}
    \def\xione{xibyxone(\xone)}
    \def\xitwo{xibyxtwo(\xtwo)}
    \def\xithree{xibyxthree(\xthree)}

    \begin{tikzpicture}[%
      xscale=1.5,
      yscale=0.7,
      declare function={%
        xibyxone(\x)   = \xizero + (\x - \xzero);
        xibyxtwo(\x)   = \xione  - (\x - \xone);
        xibyxthree(\x) = \xitwo  + (\x - \xtwo);
        xbyxione(\xi)   = \xzero + (\xi - \xizero);
        xbyxitwo(\xi)   = \xone  - (\xi - \xione);
        xbyxithree(\xi) = \xtwo  + (\xi - \xitwo);
        Ug(\x) = sin(3*\x r) - sin(3*\xzero r);
        Ua(\x,\base,\span) = 0.25*(cos(((\x-\base)/(\span))*180)-1);
        U1(\xi) = Ug(xbyxione(\xi)) + Ua(\xi, \xizero, \xione-\xizero);
        U2(\xi) = U1(\xione) - Ug(\xone) + Ug(xbyxitwo(\xi)) + Ua(\xi, \xione, \xitwo-\xione);
        U3(\xi) = U2(\xitwo) - Ug(\xtwo) + Ug(xbyxithree(\xi)) + Ua(\xi, \xitwo, \xithree-\xitwo);
      },
      ]

      \draw[->] (-0.2+\xizero, 0) -- ({\xithree+1.2}, 0) node[above] {$\xi$};

      \draw[dotted] ({\xizero},1.7)   -- ({\xizero},-1.4) node[below] {$\xi_0$};
      \draw[dotted] ({\xione},1.7)    -- ({\xione},-1.4) node[below]  {$\xi_1$};
      \draw[dotted] ({\xitwo},1.7)    -- ({\xitwo},-1.4) node[below]  {$\xi_2$};

      \draw[domain=\xizero:\xione, smooth, samples=40, gray] plot (\x, {Ug(xbyxione(\x)});
      \draw[domain=\xizero:\xione+0.58, smooth, samples=40, dashed] plot (\x, {U1(\x)}) node[right] {$U_1(x)$};
      \draw[domain=\xizero:\xione, smooth, thick, samples=40] plot (\x, {U1(\x)});

      \draw[domain=\xione:\xitwo, smooth, samples=40, gray] plot (\x, {Ug(xbyxitwo(\x))});
      \draw[domain=\xione:\xitwo, smooth, thick, samples=40] plot (\x, {U2(\x)});

      \draw[domain=\xitwo:\xithree+0.4, smooth, samples=40, gray] plot (\x, {Ug(xbyxithree(\x))}) node[right] {$U_g(\xi)$};
      \draw[domain=\xitwo:\xithree+0.5, smooth, thick, samples=40] plot (\x, {U3(\x)}) node[right] {$U(\xi)$};

      \draw ({xibyxthree(\xgoal)}, {Ug(\xgoal)}) node[vertex,fill=gray] (goalg) {};
      \draw ({xibyxone(\xgoal)}, {U1(xibyxone(\xgoal))}) node[vertex] {};
      \draw ({xibyxthree(\xgoal)}, {U3(xibyxthree(\xgoal))}) node[vertex] (goalu) {};

      \draw[<->] (goalg) -- node[right] {$U_a(\xi_*)$} (goalu);

      \draw[<->] ({\xitwo+0.53}, {Ug(xbyxithree(\xitwo+0.53))}) -- node[right] {$U_a(\xi)$} ({\xitwo+0.53}, {U3(\xitwo+0.53)});

      \tikzstyle{label}=[draw, circle,inner sep=1pt,font=\small]

      \draw[<->] ({\xizero}, 2.2) -- ({\xione}, 2.2) node[midway, above] {$\dot x > 0$} node[midway, below, label] {1};
      \draw[<->] ({\xione}, 2.2) -- ({\xitwo}, 2.2) node[midway, above] {$\dot x < 0$} node[midway, below, label] {2};
      \draw[<->] ({\xitwo}, 2.2) -- ({\xithree+1}, 2.2) node[midway, above] {$\dot x > 0$} node[midway, below, label] {3};
    \end{tikzpicture}
    \caption{The potential $U$ and $U_g$ over $\xi$, with three strokes. The
      difference is $U_a$. When enough action is applied, the goal (black dot)
      is lowered to negative potential and hence reached at positive velocity.
      In dashed we extended $U_1$ beyond the 1st stroke.}
    \label{fig:xipotential}
\end{figure}
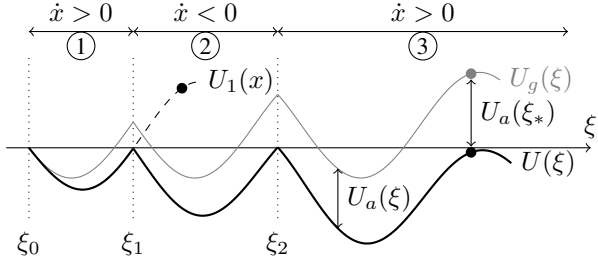

To simplify matters and to reestablish a holistic view over all strokes, we
introduce a new spatial variable $\xi$ that \enquote{unrolls} the
forth-and-back motion of $x$, like the car's odometer, by defining $\xi(t) =
x_0 + \int_0^t |\dot x(\tau)| \dif \tau$ and $\xi_i = \xi(t_i)$ analogously to
$x_i$ and $\xi_* = \xi(t_*)$. Note that $\dot \xi = |\dot x|$ and as $\xi$ is
invertible and we can uniquely reconstruct $x$ from $\xi$. This allows us to
pull over $\tilde \alpha(\xi)$, $U(\xi)$, $U_g(\xi)$ and $U_a(\xi)$ to $\xi$
with a slight abuse of notation by dropping the indices $i$. In
\cref{fig:xipotential} we visualize this construction.

To reach the goal we therefore require from the accumulation of the action
$\tilde \alpha(\xi)$ over all strokes, which is $-U_a(\xi_*)$, to overcome the
gravitational potential $U_g(x_*)$, i.e.,
\begin{align}
    \label{eq:goalxicondition}
    0 \le \amax \cdot \int_{\xi_0}^{\xi_*} \tilde \alpha(\xi) \dif \xi - U_g(x_*) - \frac{1}{2} v_*^2.
\end{align}
Note that $\tilde \alpha(\xi) \ge 0$ for the unrolled action field. We can
interpret the right-hand side of \eqref{eq:goalxicondition} as the excessive
(kinetic) energy at the goal: Equality in \eqref{eq:goalxicondition} means no
excessive kinetic energy is wasted.

\paragraph{Step 2: Unconstrained Loss Minimization}

In the next step we address the loss minimization problem in
\eqref{eq:contloss}, but without the constraints given in \eqref{eq:contopt},
except that the goal has to be reached, i.e., \eqref{eq:goalxicondition} needs
to be fulfilled. We first translate the loss from the time domain to the
unrolled spatial domain by variable substitution $t \mapsto \xi$. Using $\dot
\xi = |\dot x| = \sqrt{-2U(\xi)}$ we therefore get
\begin{align}
    \ell &= \int_0^{t_*} \alpha(t)^2 \dif t
         = \int_{\xi_0}^{\xi_*} \tilde \alpha(\xi)^2 \cdot \od{t}{\xi} \dif \xi \notag \\
         &= \int_{\xi_0}^{\xi_*} \left( \frac{\tilde \alpha(\xi)}{\sqrt{|\dot x|}} \right)^2 \dif \xi
         = \int_{\xi_0}^{\xi_*} \left( \frac{\tilde \alpha(\xi)}{\sqrt[4]{-2U(\xi)}} \right)^2 \dif \xi.
           \label{eq:lossspace}
\end{align}

One interpretation of \eqref{eq:lossspace} is that we pay less loss at higher
velocity, i.e., we yield more kinetic work of the same action at higher
velocity over a given timespan. Next we will use the following lemma for the
Hilbert space $\mathcal{L}^2([\xi_0, \xi_*])$ of square-integrable functions
over $[\xi_0, \xi_*]$, which is a consequence of the Cauchy-Schwarz inequality
(details in \cref{sec:apx-proof-csvariant}):
\begin{lemma}
    \label{lem:csvariant}
    Let $f, g \in \mathcal{L}^2([\xi_0, \xi_*])$. Then $\min_f \|f\|$ s.t.\
    $\langle f, g \rangle = 1$ is solved by $f = g / \|g\|^2$.
\end{lemma}

\begin{theorem}
    \label{thm:optalpha}
    The loss $\ell$ is minimized over all goal-reaching actions by $\alpha(t) =
    C \cdot \dot x(t)$ for some constant $C$ that fulfills
    \eqref{eq:goalxicondition} to equality.
\end{theorem}

\begin{proof}
  See \cref{sec:apx-proof-optalpha} for all details. For a brief sketch, use
  equality in \eqref{eq:goalxicondition} and apply \cref{lem:csvariant} to
  $f(\xi) = \tilde\alpha(\xi) / \sqrt[4]{-2U(\xi)}$ and $g(\xi) = \amax
  \sqrt[4]{-2U(\xi)} / \left(\frac{1}{2}v_*^2 + U_g(\xi_*)\right)$.
\end{proof}

\begin{figure}[b]
  \centering
  \includegraphics{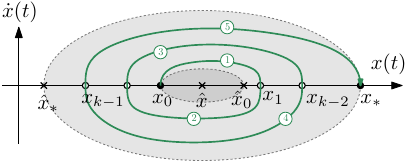}
  \caption{Trajectories in state space for the unconstrained optimization with
    $k=5$ strokes.}
  \label{fig:phasespaceunconstrained}
\end{figure}

Note that $\alpha(t) = C \cdot \dot x(t)$ has a quite simple mathematical form
and, in particular, it directly admits the policy $\pi \colon (x_t, v_t)
\mapsto C \cdot v_t$ with the given state and action spaces. It is actually
independent of $x_t$.

However, \cref{thm:optalpha} does not reveal the optimal $C$ but we have to search
for it. When we have a small $C$ then little action is applied, so for each
stroke the potential $U$ is slowly lowered and the number $k$ of strokes will
be high until the goal is eventually reached, see details in
\cref{sec:apx-unconstrainedexp}.

In \cref{fig:phasespaceunconstrained} we give a graphical interpretation of the
solutions of \cref{thm:optalpha}. The minimum of $U_g$ is denoted by $\hat{x}$
and $\hat{x}_i$ is the reflection $x_i$ at $\hat{x}$, and likewise for
$\hat{x}_*$. Assume we start at $x_0 < \hat{x}$ and do not apply any action.
Then the trajectory oscillates cyclically between $x_0$ and $\hat{x}_0$, as
indicated by the smaller ellipse-like shape. When action is applied, the
trajectory bends up, causing $x_1 > \hat{x}_0$ for the first stroke. This leads
to a spiraling trajectory ending at $x_*$ at velocity $v_* = 0$. If we increase
$C$ slightly then we reach $x_*$ with an excessive velocity $\dot x(t_*)
> v_*$. Also $x_{k-2}$ increases until $x_{k-2} = x_*$ and then we reach the
goal in only three strokes. If we reduce $C$ then the 5-th stroke does not
reach $x_*$ and we require 7 strokes to reach $x_*$, and if we reduce $C$
further we will reach $x_*$ in 7 strokes at zero velocity. If $C>0$ becomes
ever smaller then the outer ellipse will be filled with the strokes of the
trajectory, with $k, t_*, \xi_*$ tending to $\infty$.

\paragraph{Step 3: Reestablishing the Constraints}

In the final step we reestablish the constraints given in \eqref{eq:contopt}.
First we note that $|\dot x| \le \vmax$ and $x \in [\xmin, \xmax]$ is enforced
by the \gls{rl} environment. If $\dot x$ or $x$ would leave their respective
interval then they are limited. And when $x$ is limited, also $\dot x$ is set
to zero, i.e., as an inelastic bump into a wall.
We distinguish two cases:
\begin{enumerate}

  \item The trajectory does not reach the left wall. We call this a
    \emph{single-phase trajectory}.

  \item The trajectory reaches the left wall. We call this a \emph{two-phase
    trajectory}.

\end{enumerate}
In either case we identify a finite number of candidate trajectories, depending
on the value of $k$, and the optimal solution is the one with the smallest
loss.
In the unconstrained case we see that the smaller $C$ the smaller $\ell$, but
once we hit constraints this will become unclear. We first note that we have $k
\ge 2$ as $k=1$ is impossible by \eqref{eq:goalxicondition} with the given
constants
. Hence,
we have an upper bound $\Cmax$ due to $k \ge 2$. Also note that there is
a smallest $\Cmin$ such that for all $C \ge \Cmin$ we respect $t_* \le \tmax$.
As a consequence, there is an upper bound for $k$, which we denote by $\kmax$.

\emph{Single-phase trajectories.}
Consider a $C \in [\Cmin, \Cmax]$. Assume we reach $x_*$ at some speed $\dot
x(t_*)$ in $k$ strokes. When we reduce $C$ then we reduce the goal velocity
$\dot x(t_*$), increase $t_*$, but we also reduce $\xi_* - \xi_0$ and hence
$\ell$. Let us reduce $C$ until $\dot x(t_*) = v_*$ or $t_* = \tmax$, while
leaving $k$ unchanged. This yields an $\ell$-optimal $C_k$ respecting $t_* \le
\tmax$ for all $2 \le k \le \kmax$. Figuratively speaking, we can think of
starting with any $C$ large enough such that $k=2$, and then reduce $C$ down
until $\Cmin$ while discovering all $C_k$. If $\min_t x(t) = x_{k-1}$ becomes
less than $\xmin$ after a sufficiently large $k$ then these single-phase
solutions are not feasible.

\emph{Two-phase trajectories.}
\label{sec:analysis-twophase}
Assume $x(t)$ hits $\xmin$ at time $\tau$. Then the state $(x(\tau), \dot
x(\tau))$ is reset to $(\xmin, 0)$ by the environment. This reset separates the
future from the past, so we can discuss them independently, i.e., how to reach
$\xmin$ and how to continue to $x_*$ in an $\ell$-optimal fashion. This
establishes the \emph{phase~1} until $\tau$ and \emph{phase~2} after $\tau$.
Observe that we have to reach the goal $x_*$ in a single stroke in phase~2,
otherwise we would just reenter state $(\xmin, 0)$ later on.

Roughly speaking, for any given number $k$ of strokes, we search for the
optimal $C_{1,k}$ for phase~1 comprising $k-1$ and the optimal $C_{2,k}$ for
phase~2. In \cref{fig:phasespace-leftwall} we have a green phase~2 trajectory
that reaches $x_*$ at velocity $v_*$, which is zero. All other phase~2
trajectories have excessive velocity and reside above this trajectory. Hence,
the state space is separated into the green shaded area of phase~2 trajectories
and the complement, in which the phase~1 trajectories reside. This eventually
allows us to formulate a single policy $\pi \colon (x, \dot x) \mapsto \alpha$
also for two-phase trajectories in the next section.
A full discussion of two-phase trajectories and the different cases is given in
\cref{sec:apx-twophasetrajectories}.

\begin{figure}[b]
  \centering
  \includegraphics{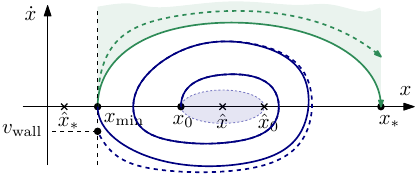}
  \caption{Phases~1 and 2. In the green area phase~2 trajectories can reside
  in. The green line gives the optimal phase~2 trajectory, and bounds the green
area. The dashed trajectory has $\vwall \neq 0$.}
  \label{fig:phasespace-leftwall}
\end{figure}

\section{Optimal Policy for Mountain Car}

\paragraph{Single- versus Two-Phase Trajectories}

A numerical search revealed that single-phase trajectories do exist. For
instance, for $x_0 = -0.6$ we obtain the optimal $C=\num{4.891}$, leading to
a loss of $\ell = \num{5.354}$, hence a of return of $\num{99.46}$, and a goal
velocity of zero, while not reaching the left wall.
However, the best two-phase trajectories from $x_0 = -0.6$ has a slightly
better return of $\num{99.63}$. We experimentally found that single-phase
trajectories are sub-optimal for all starting points $x_0$, limiting our
discussion to two-phase policies.
Furthermore, a search for the optimal two-phase trajectories showed that for
each $x_0$ the optimum is achieved by having $k$ such that the left wall is
reached at velocity zero and the goal at velocity zero. We determined $C_{2,k}
= \num{4.8358}$, while $C_{1,k}$ depends on $x_0$. The constraint $\vmax$ was
never hit and can therefore be ignored for the optimal control.

\paragraph{The Analytical Worst-Case Policy}

The training goal of \gls{rl} agents follows a \gls{mdp} regime and hence we
maximize the expected return, i.e., consider $\max_\pi \mathbb{E}_{x_0 \sim
\mathcal{U}([-0.6, -0.4])}(R)$. But since the penalty of not reaching the goal
is more than two magnitudes higher than the regret we have achieved, for any
given $C_{1,k}$ the measure of the set of $x_0$ where the goal is not reached
must be correspondingly small. Furthermore, \eqref{eq:goalxicondition}
essentially tells us that the worst-case effort to reach the goal is given for
$x_0$ chosen at the vicinity of $\hat{x} = -\frac{\uppi}{6} \approx -0.524$.
This motivates us to consider $\min_{x_0} \max_\pi R$, which we find
approximately at $x_0 = \hat{x}$, and call it the \emph{analytical worst-case
policy} $\piana$.

Note that for $x_0 = \hat{x}$ we have $U'_g$ being zero and hence no action is
ever applied. An arbitrarily small initial action has to be applied to
\enquote{bootstrap} the motion. To overcome this hurdle, also in the presence
of numerical inaccuracies, we apply a minimum action $\alpha_\text{boot}$ in
a vicinity of $\hat{x}$. This leads to the following analytical worst-case
policy
\begin{align}
  \label{eq:poly_incl_worst}
  \piana(x, v) = \sign(v) \cdot \max(C(x, v) |v|, \alpha_\text{boot}(x, v))
\end{align}
with $C(x, v) = 4.8358$ in the phase~2 area of the state space and $C(x, v)=
4.3346$ otherwise. We heuristically set $\alpha_\text{boot}(x, v) = 0.1$ for
$|x - \hat{x}| \le 0.01$ and zero otherwise. Note that the concrete values for
$\alpha_\text{boot}$ have small impact on $\ell$ given that only a little
fraction of the state space (cf.~\cref{fig:policycomparison}) is affected and
the squared actions in $\ell$ are small. Hence, we refrain from further
discussions in this paper.
For $\piana$ we achieve a return $R$ of $\numrange{99.15}{99.52}$ for $x_0 \in
[-0.6, -0.4]$, see \cref{tab:evalpolicies} later.

Although $\piana$ achieves a return close to the upper bound, we can still ask
for a possible gap to the original discrete formulation. For this reason we set
up a discrete optimization task using the \emph{fmincon} optimizer of Matlab
R2024b and confirmed that no better solution could be found. (For the optimizer
to converge, it needed to be initialized based on the continuous optimum.)
See \cref{sec:apx-optimality-discrete-problem} for details.

Based on our analysis we propose two variants of the Mountain Car problem: By
moving the left wall $\xmin$ more to the left, at some point this would change
the optimal policy from two- to one-phase trajectories. Similarly, the
constraint $\vmax$ of $0.07$ can currently be ignored, but reducing it to
$0.05$ would make this constraint kick in and increase the richness of this
benchmark task.

\paragraph{Regret of the SOTA on Mountain Car}
\label{sec:regret-on-sota}

The analytical worst-case policy $\piana$ has a return between \num{99.15} and
\num{99.52}, and hence is quite close to the trivial upper bound of \num{100}.
\Cref{tab:evalpolicies} lists the leading agents from RL Baselines3 Zoo. The
\gls{ars} agent leads with a return $R$ between \num{92.51} and \num{97.41} on
different starting points $x_0$. The mean regret $r$ of \gls{ars} is therefore
\num{2.72} and appears surprisingly high, not speaking of \gls{ppo} with a regret of
\num{5.48}.

Possible explanations could be limitations in (i) the exploration phase, (ii)
the \gls{rl} algorithms or (iii) the (parametrized) policy space.
From \cref{sec:analyticalsolution} we infer that a goal-reaching agent necessarily
implements some form of \enquote{spiraling} trajectories in state space qua
dynamics of the task, also cf.~\cref{fig:policycomparison}: The state space is
necessarily explored well, once the goal is reached.
In \gls{rl} Baselines3 Zoo, we find a large diversity of algorithms, trained
with proper hyper parameter tuning and sufficient training time. On the other
hand, in \cite{Rajeswaran2017} it was demonstrated that for many well-known
continuous control tasks, simple policies are beneficial to avoid overfitting
models, yet Mountain Car was not investigated. Beyond the study
\cite{Rajeswaran2017}, we simply see from \cref{eq:poly_incl_worst} and the
illustration in \cref{fig:policycomparison} that simple approximators should
suffice.
Hence, we ran \gls{ppo} with smaller neural nets\footnote{Default is two layers of
size 64, i.e., $[64, 64]$. We also ran it with $[32, 32], [16, 16], [64], [32]$
and $[16]$.}. However, this led to degraded performance, see
\cref{sec:apx-ppo-on-mlp-mountaincar} for details. This motivated us to look at
fundamentally different approximator models from first principles.

\section{Chebyshev Policies}

\subsection{Multi-Variate Chebyshev Polynomials}

The tameness of the optimal solution, as in \cref{eq:poly_incl_worst}, appears
to be typical for low-dimensional control tasks, given the general nature of
the underlying \gls{ode}.
We put this principle first: A class of functions that are universal, versatile
and efficient at sampling the space of continuous policies. Polynomials are
versatile and an orthogonal basis spanning a dense subset of the continuous
policy space provides us with universality and efficiency.

This motivated us to generalize Chebyshev polynomials to a multi-variate
setting $\nR^n \to \nR^m$ and provide them as a parametrized model in auto-diff
frameworks like PyTorch as drop-in replacements for neural nets for stochastic
policies. Without loss of generality, we consider $m=1$, since we can setup
a multi-variate polynomial for each factor $\nR$ of the co-domain $\nR^m$.

For $(x_1, \dots, x_n) \in \nR^n$ the degree of the monomial $\prod_{i=0}^n
x_i^{d_i}$ is defined in literature as $\deg f = \sum_i d_i$. A polynomial is
a linear combination of monomials and its degree is the maximum degree of its
monomials. For our purpose it will be convenient to introduce the notion of the
\emph{max-degree} $\deg^* f = \max_i d_i$ of a monomial and polynomial. Note
that the space of multi-variate polynomials $f$ with  $\deg^* f \le d$ forms
a vector space, which we denote by $P^n_d$. Also $\{\prod_{i=0}^n x_i^{d_i}
\colon 0 \le d_i \le d \}$ forms a basis of $P^n_d$, which we call the
\emph{power basis}, and hence $\dim P^n_d = (d+1)^n$.

Let $T_k \colon [-1, 1] \to \mathbb{R} \colon x \mapsto \cos(k \cdot
\arccos(x))$ denote the $k$-th Chebyshev polynomial of the first kind, having
$\deg T_k = k$.
Let us denote by $\mathcal{L}^2_w([-1,1])$ the Hilbert space of square-integrable
functions $[-1, 1] \to \nR$ with the weighted inner product
$
  \langle f,g \rangle_w = \int_{[-1,1]^n} f(x) g(x) \; w(x) \mathrm{d}x,
  $
and the weight function $w(x) = (1-x^2)^{-1/2}$. It is well known that the
$T_0, \dots, T_d$ form an orthonormal basis for $P_d^1$ in
$\mathcal{L}^2_w([-1,1])$, making them favorable in approximation theory,
together with the property that their minima and maxima in $[-1,1]$ have
absolute value $1$.
Next, we generalize to $n$-variate Chebyshev polynomials and refer to
\cite{Zeller1966} for an early work on multi-variate Chebyshev polynomials and
to \cite{Dressler2024} for a more recent work on approaches for this
generalization. In this paper, we use
\begin{align}
  T_{d_1, \dots, d_n}(x_1, \dots, x_n) = \prod_i T_{d_i}(x_i)
\end{align}
and note that $\deg^* T_{d_1, \dots, d_n} = \max_i d_i$.
It is easy to check that the weighted orthonormality generalizes to $\langle
T_{d_1, \dots, d_n}, T_{u_1, \dots, u_n} \rangle_w = \prod_i \delta_{d_i,u_i}$
over $[-1,1]^n$, where $\delta_{d_i,u_i}$ denotes the Kronecker delta and the
weight function is generalized to $w(x_1, \dots, x_n) = \prod_{i=1}^n w(x_i)$.
The $n$-variate Chebyshev polynomials also have their absolute extreme values
at $1$.
Note that there are $(d+1)^n$ linearly independent $n$-variate Chebyshev
polynomials of max-degree at most $d$, and hence they form an orthonormal basis
of $P^n_d$ again. That is, given a function $f \colon [-1,1]^n \to \mathbb{R}$
we can uniquely approximate $f$ by an element in $P^n_d$ via a linear
combination
\begin{align}
  f \approx \sum_{i_1, \dots, i_n = 0, \dots, 0}^{d, \dots, d} \theta_{i_1, \dots, i_n} T_{i_1, \dots, i_n}.
\end{align}
The coefficients $\theta_{i_1, \dots, i_n}$ could be determined by the
previously explained inner product directly, but we implemented this model type
in PyTorch, which allows us to facilitate them in a learning loop.

\subsection{Chebyshev Polynomials for Stochastic Policies}

We use multi-variate Chebyshev polynomials to form stochastic policies, i.e.,
in a given state $s$ we have a parametrized distribution $\pi_\theta(s)$ over
the action space, with a parameter vector $\theta$, and we draw an action
$\alpha \sim \pi(s)$. A typical implementation, like for the Mountain Car
problem, is to set $\pi_\theta(s) = \mathcal{N}(\mu_\theta(s),
\sigma_\theta(s))$, i.e., a normal distribution where the mean and standard
deviation depend on the state. The parametrized models $\mu_\theta$ and
$\sigma_\theta$ are usually implemented as a neural net of some architecture,
translating $s \mapsto (\mu(s), \sigma(s))$, and $\theta$ would be the joint
parameter vector of both nets.

By \emph{Chebyshev policies} we mean that $s \mapsto (\mu(s), \sigma(s))$ is
modeled by two polynomials of max-degree $d$ using multi-variate Chebyshev
polynomials as basis, i.e.,
\begin{align}
  \label{eq:cheby-mu}
  \mu(s)    = \sum_{i_1, \dots, i_n = 0, \dots, 0}^{d, \dots, d}
            \theta^{(\mu)}_{i_1, \dots, i_n} T_{i_1, \dots, i_n}(s)
\end{align}
and likewise for $\sigma(s)$ with $\theta^{(\sigma)}_{i_1, \dots, i_n}$.
This setup makes Chebyshev policies available to all sort of policy gradient
algorithms, e.g., classical REINFORCE, but also modern algorithms like \gls{trpo} and
\gls{ppo}.
To this end, we implemented multi-variate Chebyshev polynomials in PyTorch in
a similar manner as \cite{Waclawek2024} did for the univariate case.
In case of \gls{ppo}, we also model the critic $v_\pi(s)$ by a Chebyshev polynomial, and
hence the joint parameter vector encompasses three multi-variate polynomials.
Vice versa, for \gls{ars} there is no $\sigma$ and we only have to train $\mu$.

Our practical evaluation revealed that choosing a lower max-degree $d \le 3$
for $\sigma$ was beneficial for the training dynamics. The $\sigma$-polynomial
is initialized to $1$
and all other polynomials are initialized with small random coefficients in the
range $\pm 10^{-3}$.

\section{Evaluation}
\label{sec:evaluation}

All experimental results were created using PyTorch version $2.4.0$, Gymnasium
$1.0.0$, and Stable Baselines3 with RL Baselines3 Zoo versions
$2.4.0$. 
The experiments discussed in the following were conducted on the
\emph{MountainCarContinuous-v0} environment in Gymnasium.

\begin{table}[b]
  \centering
  \caption{Performance on Mountain Car. Mean return $\overline{R}$ (regret
  $r$), min and max return, mean time $t_*$ to goal and $\|\pi - \piana\|_2$.}
  \label{tab:evalpolicies}
  \small
  \begin{tabular}{llccc}
    \toprule
    Policy~$\pi$         & $\overline{R}\uparrow$ ($r\downarrow$) & min $R$ -- max $R$ & $t_*\uparrow$ & $\|.\|_2\downarrow$ \\ \midrule
    $\piana$           & 99.39                                  & 99.15   -- 99.52   & 769           &  \\ \midrule
    \textsc{CH-3-ARS}  & 98.95 (0.44)                           & 98.74   -- 99.11   & 471           & 0.152 \\
    \textsc{CH-3-REI}  & 98.62 (0.77)                           & 98.31   -- 98.89   & 396           & 0.068 \\
    \textsc{CH-3-PPO}  & 98.10 (1.29)                           & 97.61   -- 98.42   & 469           & 0.087 \\
    \textsc{ARS}       & 96.67 (2.72)                           & 92.51   -- 97.42   & 239           & 0.211 \\
    \textsc{SAC}       & 94.61 (4.78)                           & 89.70   -- 95.77   & 106           & 0.317 \\
    \textsc{PPO}       & 93.91 (5.48)                           & 90.86   -- 95.23   & 298           & 0.273 \\
    \bottomrule
  \end{tabular}
\end{table}

\paragraph{Do Chebyshev Policies Improve upon MLP Policies?}
\label{sec:eval-cheby-beats-mlp}

We evaluate agents as follows: We choose \num{100} evenly spaced $x_0$ from
$[-0.6, -0.4]$ and record the achieved returns $R$. Note that the mean return
$\overline{R}$ is an estimator for the expected return $\mathbb{E}_{x_0 \sim
\mathcal{U}([-0.6, -0.4])}(R)$.
Further details on the training and evaluation protocol can be found in
\cref{sec:details-eval-mountaincar}.

For neural policies we took the pretrained RL Baseline3 Zoo agents on the
MountainCarContinuous-v0 problem \cite{Raffin2024}. The best performers were
trained by \gls{ars}, \gls{sac} and \gls{ppo}. We furthermore trained Chebyshev policies of
max-degree 3 with \gls{ars}, \gls{ppo} and the classical REINFORCE algorithm, which we call
\textsc{CH-3-ARS}, \textsc{CH-3-PPO} and \textsc{CH-3-REI}. In
\cref{tab:evalpolicies} we report on the results and discuss them as follows.
(Further details can be found in \cref{sec:details-eval-mountaincar}.)

First of all, every Chebyshev policy trained by different algorithms
significantly outperforms the neural net policies by a large margin in terms of
regret. The best neural policy (\textsc{ARS}) achieved a regret of \num{2.72},
while \textsc{CH-3-ARS} improved it to \num{0.44}, which is a factor of \num{6.18}.
We note that for both models \textsc{ARS} leads to the best performer. But also
with \gls{ppo}, Chebyshev policies improve the regret to \num{1.29} from \num{5.48},
which is again a factor of \num{4.24}.

We actually tried to find a bad performing Chebyshev policy by implementing the
classic REINFORCE algorithm. Interestingly, while REINFORCE fails in obtaining
any goal-reaching neural policy (see \cref{sec:apx-reinforce-fails-mlp-mountaincar}), it
succeeds with Chebyshev policies, actually outperforming all other neural net
policies by a large margin, namely improving the regret to \num{0.77} compared
to \num{2.72} of \textsc{ARS}, which is a remarkable factor of $3.53$.

\begin{figure*}[t]
  \centering
  \includegraphics[width=\textwidth,trim={3.6mm 3mm 4mm 2.5mm},clip]{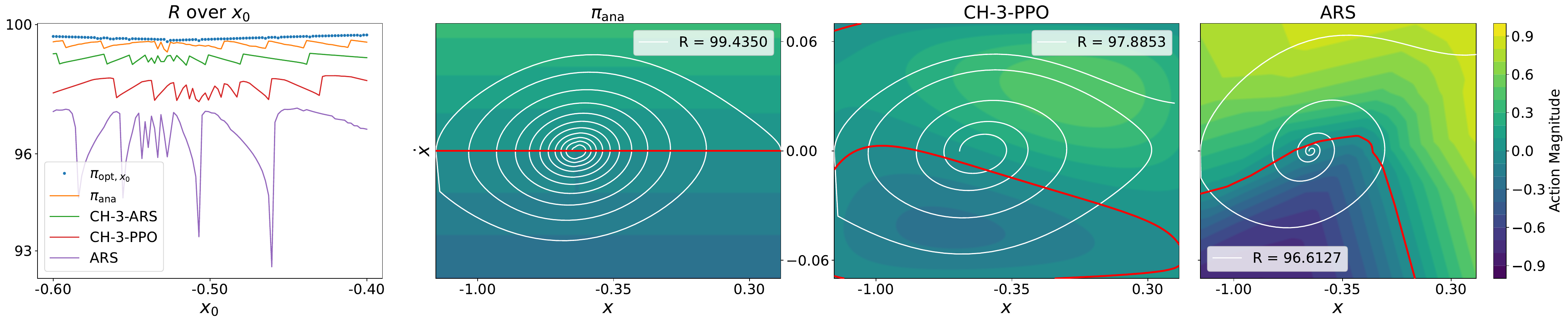}
  \caption{The left figure plots $R$ for each policy. The right figures plot
    the actions of $\piana$, \textsc{CH-3-PPO} and \gls{ars} over the state space, the
  zero-actions in red and in white a trajectory from $x_0 = -0.55$.}
  \label{fig:policycomparison}
\end{figure*}

\paragraph{How do the Control Strategies Compare?}
\label{sec:control-strategies-compare}

In knowledge of the optimal control, we can investigate deeper why and when the
Chebyshev policies outperform the neural policies. In
\cref{fig:policycomparison}, in the left subfigure, we plot the return of \gls{ars}
(best performing neural policy) and \textsc{CH-3-PPO} (worst performing
\textsc{CH-3}) over the different start positions $x_0$, and hence attain
a deterministic point of view. Once we fix $x_0$, we could ask what the optimal
control would be given $x_0$, which we denote by $\pioptxzero$.
Interestingly, the return of \gls{ars} is highly sensitive to $x_0$, dropping
down to \num{92.51}. In contrast, Chebyshev policies give a stable performance
relative to $\piana$.

In the three right subfigures, we plot the actions over the state space and
a sample trajectory. Here we can observe a key deficiency of \textsc{ARS}: it
outputs positive action for negative $\dot x$ in a large area of the state
space, which counteracts the dynamics of Mountain Car.
\textsc{ARS} outputs large actions compared to $\piana$, causing the car to
reach the goal in $t_* = 239$ steps on average, lowering the return and
increasing excess goal velocity, see \cref{tab:evalpolicies}.
In contrast, \textsc{CH-3-PPO} is closer to $\piana$ than \gls{ars}, which we can
summarize by looking at $\|\pi - \piana\|_2$ over the state space $[-1.2, 0.45]
\times [-0.07, 0.07]$ for different polices $\pi$. See
\cref{sec:apx-indepth-policies-mountaincar} for more details.

\paragraph{Do Chebyshev Policies Perform Well on Other Tasks?}
\label{sec:eval-cheby-other-tasks}

Although Mountain Car forms a dynamical system typical for low-dimensional
control tasks in engineering, we still may ask how a different dimensionality
of the state space or a change of the underlying system dynamics
affects the performance of Chebyshev policies. We therefore also evaluated
Chebyshev policies on the \emph{Pendulum-v1} environment in Gymnasium.
%
%
Its (normalized) state space $S^1 \times [-1,1] \hookrightarrow \nR^3$ is
3-dimensional and formed by an angular position as a Cartesian point on the
unit circle $S^1$ and the angular velocity. The action is a scalar torque
applied to the pendulum.
We pick $50 \times 50$ equidistant start states from $S^1 \times [-1, 1]$ and
report on the average return in \cref{tab:evalenvironments}. The return
captures a slightly different underlying objective than Mountain Car.
We report on the results in \cref{tab:evalenvironments} and observe that for
both \gls{ars} and \gls{ppo}, Chebyshev policies again clearly outperform the pretrained RL
Baselines3 Zoo agents. For the Chebyshev policies, our detailed evaluation
shows that max-degree 6 worked best for \gls{ars} and max-degree 5 for \gls{ppo}.

\begin{table}[b]
  \centering
  \caption{Average return on Mountain Car (MC), Pendulum and Aero~2 environments in
  simulation and real world.}
  \label{tab:evalenvironments}
  \small
  \begin{tabular}{llrrrr}
    \toprule
                    & MC    & Pendulum & Aero~2 sim & Aero~2 real \\ \midrule
    \textsc{CH-ARS} & 98.95 & -150.8   & -125.2     & -164.2      \\
    \textsc{ARS}    & 96.67 & -218.3   & -721.8     & -718.4      \\ \midrule
    \textsc{CH-PPO} & 98.10 & -162.8   & -49.2      & -55.8       \\
    \textsc{PPO}    & 93.91 & -176.2   & -84.6      & -182.0      \\
    \bottomrule
  \end{tabular}
\end{table}

The transition from simulation to real-world environments comes at well-known
\gls{rl} challenges, which motivated us to evaluate Chebyshev policies on the
Quanser Aero~2 helicopter testbed. 
Following \cite{Schaefer2024a}, we have a 3-dimensional state space,
1-dimensional action space and significant non-linear motion dynamics. In
\cref{tab:evalenvironments} we report returns gathered from an Aero~2
simulation and real-world operation of the same policies and the task is to
follow a target pitch signal.
We again see that all Chebyshev policies significantly outperform their \gls{mlp}
counterparts. A closer look reveals that \textsc{ARS} fails to follow the
target signal. Furthermore, \textsc{PPO} resorts to a bang-bang control,
leading to excessive power consumption and motor heat in our experiments.

When investigating the performance of the policies on the real system, we see
that (i) all Chebyshev policies outperform all \gls{mlp} policies and (ii) the
Chebyshev policies better retain the simulation performance. The appendix
provides more details on both environments in \cref{sec:apx-further-envs}.

\paragraph{Limitations}
\label{sec:eval-cheby-dimensionality}

A limitation of Chebyshev policies can be found in the combinatorial growth of
the number $(d+1)^n$ of basis polynomials. Hence, for large numbers, this
growth leads to high computational demands and at some point it will impair
return, training stability and numerical stability.
Although many industrial control tasks deal with small $n$, modern control
challenges, like in humanoid robotics, deal with $n$ in the tens or hundreds.
To get a preliminary impression on computational costs, training stability and
numerical limitations for $1 \le d \le 50$ at $n=2$ on the Mountain Car
environment, we conducted experiments in \cref{sec:apx-dimensionality}.

Another limitation can rise from the uniform convergence of Chebyshev policies.
While we prefer uniform convergence in general, it can be a weakness when the
state space should not be treated equal everywhere. In contrast, the
hierarchical architecture of \glspl{mlp} allows to specifically focus on
important regions and concentrate model capacity there.
This allows \glspl{mlp} to implement very different behavior at different
locations, like high variety of actions at one part of the state space and low
variety at other parts.

\section{Conclusion}

Reinforcement Learning is broadly applied, from particle accelerators and
robotic control to large language models, yet we also face long-standing gaps
in \gls{rl} theory.
After 36 years, we analytically solved Mountain Car to optimality. It led to
surprising insights concerning the high regret of modern \gls{rl} agents. This
raises pressing research questions of general nature, which we underpinned with
rigor evaluation on different environments.

We introduced a new class of Chebyshev policies from ground up as drop-in
replacements of \gls{mlp} policies from first principle: forming a universal
subspace of the continuous policy space.
They reduce the regret \num{6.18}-fold, come close to optimality, and among three
algorithms (\gls{ars}, REINFORCE, \gls{ppo}), Chebyshev policies consistently
and significantly outperformed \gls{mlp} policies, while having $277\times$
less model parameters\footnote{When compared to a \gls{mlp} with layer sizes 2,
64, 64, and 1, which has $2 \cdot 65 + 64 \cdot 65 + 65 = 4355$ parameters.}.
Using a $n$-variate Horner scheme, max-degree $d$ polynomials can be evaluated
with $(d+1)^n$ multiplications and additions, fostering real-time performance,
especially for embedded devices without neural accelerators.
We also find that on the Pendulum environment and the Aero~2 real-world
testbed, Chebyshev policies consistently improve upon \gls{mlp} policies,
improve the sim-to-real transfer and dynamical stability.

The uniform convergence of Chebyshev polynomials might also be a limitation
when compared to \glspl{mlp}: Especially with ReLU activation and multiple
layers, a \gls{mlp} can erect a policy that has a hierarchical structure that
adapts to different regions of the state space. Optimal controls in resemblance
to bang-bang or sliding-mode controllers are therefore better achieved by
\glspl{mlp}.

We anticipate numerical limitations at higher polynomial degrees and dimension,
and hence our findings are currently targeted at low-dimensional control tasks.
However, we suggest a future research direction on a hybrid approach to combine
\glspl{mlp} with Chebyshev representation layers, where \glspl{mlp} and
Chebyshev policies form special cases.
Finally, we hope that these results encourage deeper integration of analytical
insight into \gls{rl}, as such understanding may unlock further real-world
impact.

\section*{Acknowledgment}

\ifdefined\isaccepted
  The financial support by the Christian Doppler Research Association, the
  Austrian Federal Ministry for Digital and Economic Affairs, the European
  Interreg Österreich-Bayern project BA0100172 AI4GREEN and the WISS project of
  the Federal State of Salzburg is gratefully acknowledged.
\else
  Left for initial submission.
\fi

\section*{Impact Statement}




This research contributes to the development of transparent and computationally
efficient \gls{rl} methods. By replacing neural black-box function
approximators with analytic polynomial models, it supports decision-making and
control systems that are potentially easier to audit, verify, and trust. These
qualities are essential for deploying \gls{rl} in safety-critical domains such
as robotics, energy management, and autonomous systems.

\bibliography{mountain-car-analytic-polynomial-rl-approximator}

@inproceedings{Schaefer2024a,
  title         = {Comparison of Model Predictive Control and Proximal Policy Optimization for a {1-DOF} Helicopter System},
  booktitle     = {2024 {IEEE} 22nd {IEEE} International Conference on Industrial Informatics {(INDIN'24)}},
  author        = {Sch{\"a}fer, Georg and Rehrl, Jakob and Huber, Stefan and Hirlaender, Simon},
  year          = {2024},
  publisher     = {{IEEE}},
  address       = {{Beijing, China}},
  doi           = {10.1109/INDIN58382.2024.10774357}
}

@inproceedings{Schaefer2024b,
  title         = {{Python-Based Reinforcement Learning on Simulink Models}},
  author        = {Sch{\"a}fer, Georg and Schirl, Maximilian and Rehrl, Jakob and Huber, Stefan and Hirlaender, Simon},
  year          = {2024},
  booktitle     = {11th International Conference on Soft Methods in Probability and Statistics {(SMPS 2024)}},
  address       = {Salzburg, Austria},
  note          = {accepted},
  doi           = {10.1007/978-3-031-65993-5_55}
}

@misc{Gym-mountaincarcontinuous,
  key           = {Gym-mcc},
  title         = {{Gymnasium: Mountain Car Continuous}},
  url           = {https://gymnasium.farama.org/environments/classic_control/mountain_car_continuous/},
  urldate       = {21.01.2026},
  year          = 2024
}

@misc{Gym-leaderbord,
  key           = {Gym-lb},
  title         = {{OpenAI Gym}: Leaderboard},
  url           = {https://github.com/openai/gym/wiki/Leaderboard},
  urldate       = {21.01.2026},
  year          = 2024
}

@misc{Raffin2024,
  author        = {Raffin, Antonin and Simonini, Thomas and Ernestus, Maximilian and Gallou\'{e}dec, Quentin},
  title         = {Reinforcement Learning models trained using {Stable Baselines3} and the {RL Zoo}.},
  year          = {2024},
  publisher     = {Hugging Face},
  howpublished  = {\url{https://huggingface.co/sb3}},
  urldate       = {21.01.2026}
}

@misc{Mania2018,
  title         = {Simple Random Search Provides a Competitive Approach to Reinforcement Learning},
  author        = {Horia Mania and Aurelia Guy and Benjamin Recht},
  year          = {2018},
  eprint        = {1803.07055},
  archiveprefix = {arXiv},
  url           = {https://arxiv.org/abs/1803.07055}
}

@phdthesis{Moore1990,
  title         = {Efficient memory-based learning for robot control},
  author        = {Moore, Andrew William},
  school        = {University of Cambridge},
  year          = {1990},
  url           = {https://www.cl.cam.ac.uk/techreports/UCAM-CL-TR-209.pdf}
}

@inproceedings{Rajeswaran2017,
  author        = {Rajeswaran, Aravind and Lowrey, Kendall and Todorov, Emanuel V. and Kakade, Sham M},
  booktitle     = {Advances in Neural Information Processing Systems {(NIPS 2017)}},
  publisher     = {Curran Associates, Inc.},
  title         = {Towards Generalization and Simplicity in Continuous Control},
  volume        = {30},
  year          = {2017}
}

@inproceedings{Waclawek2024,
  title         = {Machine Learning Optimized Orthogonal Basis Piecewise Polynomial Approximation},
  author        = {Waclawek, Hannes and Huber, Stefan},
  year          = 2024,
  month         = jun,
  booktitle     = {Learning and Intelligent Optimization {(LION 18)}},
  series        = {Lecture Notes in Computer Science},
  address       = {Ischia, Italy},
  publisher     = {Springer Cham},
  doi           = {10.1007/978-3-031-75623-8_33}
}

@book{Koenigsberger2000,
  title         = {Analyis 1},
  author        = {K\"{o}nigsberger, Konrad},
  edition       = 6,
  year          = 2004,
  publisher     = {Springer},
  isbn          = {978-3-642-18490-1}
}

@book{Sutton2018,
  author        = {Sutton, Richard S. and Barto, Andrew G.},
  edition       = {Second},
  publisher     = {The MIT Press},
  title         = {Reinforcement Learning: An Introduction},
  url           = {http://incompleteideas.net/book/the-book-2nd.html},
  year          = {2018}
}

@inproceedings{Grande2014,
  title         = {Sample Efficient Reinforcement Learning with {Gaussian} Processes},
  author        = {Grande, Robert and Walsh, Thomas and How, Jonathan},
  booktitle     = {Proceedings of the 31st International Conference on Machine Learning},
  pages         = {1332--1340},
  year          = {2014},
  volume        = {32},
  number        = {2},
  series        = {Proceedings of Machine Learning Research},
  address       = {Bejing, China},
  month         = jun,
  publisher     = {PMLR}
}

@inproceedings{Schulman2015,
  title         = {Trust Region Policy Optimization},
  author        = {Schulman, John and Levine, Sergey and Abbeel, Pieter and Jordan, Michael and Moritz, Philipp},
  booktitle     = {Proceedings of the 32nd International Conference on Machine Learning {(ICML 2015)}},
  pages         = {1889--1897},
  year          = 2015,
  editor        = {Bach, Francis and Blei, David},
  volume        = {37},
  series        = {Proceedings of Machine Learning Research},
  address       = {Lille, France},
  month         = jul
}

@inproceedings{Haarnoja2018,
  title         = {Soft Actor-Critic: Off-Policy Maximum Entropy Deep Reinforcement Learning with a Stochastic Actor},
  author        = {Haarnoja, Tuomas and Zhou, Aurick and Abbeel, Pieter and Levine, Sergey},
  booktitle     = {Proceedings of the 35th International Conference on Machine Learning {(ICML 2018)}},
  pages         = {1861--1870},
  year          = 2018,
  volume        = {80},
  series        = {Proceedings of Machine Learning Research},
  month         = jul,
  publisher     = {PMLR}
}

@misc{Schulman2017,
  title         = {Proximal Policy Optimization Algorithms},
  author        = {John Schulman and Filip Wolski and Prafulla Dhariwal and Alec Radford and Oleg Klimov},
  year          = {2017},
  eprint        = {1707.06347},
  archiveprefix = {arXiv},
  primaryclass  = {cs.LG},
  url           = {https://arxiv.org/abs/1707.06347}
}

@article{Zeller1966,
  author        = {Zeller, Karl and Ehlich, Hartmut},
  journal       = {Mathematische Zeitschrift},
  pages         = {142--143},
  title         = {{Cebysev-Polynome} in mehreren {Ver\"{a}nderlichen}.},
  url           = {http://eudml.org/doc/170581},
  volume        = {93},
  year          = {1966}
}

@misc{Dressler2024,
  title         = {Optimization-Aided Construction of Multivariate {Chebyshev} Polynomials},
  author        = {Dressler, Mareike and Foucart, Simon and Joldes, Mioara and de Klerk, Etienne and Lasserre, Jean Bernard and Xu, Yuan},
  year          = 2024,
  eprint        = {2405.10438},
  archiveprefix = {arXiv},
  primaryclass  = {math.OC},
  url           = {https://arxiv.org/abs/2405.10438}
}

@article{DulacArnold2021,
  author        = {Dulac-Arnold, Gabriel and Levine, Nir and Mankowitz, Daniel J. and Li, Jerry and Paduraru, Cosmin and Gowal, Sven and Hester, Todd},
  title         = {Challenges of real-world reinforcement learning: definitions, benchmarks and analysis},
  journal       = {Machine Learning},
  year          = 2021,
  volume        = 110,
  number        = 9,
  pages         = {2419--2468},
  doi           = {10.1007/s10994-021-05961-4}
}

@book{HandFinch1998,
  author        = {Hand, Louis N. and Finch, Janet D.},
  title         = {Analytical Mechanics},
  publisher     = {Cambridge University Press},
  year          = {1998},
  address       = {Cambridge},
  isbn          = {978-1139639514}
}

@misc{Eysenbach2023,
  title         = {A Connection between One-Step Regularization and Critic Regularization in Reinforcement Learning},
  author        = {Eysenbach, Benjamin and Geist, Matthieu and Levine, Sergey and Salakhutdinov, Ruslan},
  year          = 2023,
  month         = jul,
  booktitle     = {Proceedings of the 40th International Conference on Machine Learning {(ICML 2023)}},
  page          = {9485--9507}
}

@inproceedings{Tang2025,
  author        = {Tang, Chen and Abbatematteo, Ben and Hu, Jiaheng and Chandra, Rohan and Mart\'{\i}n-Mart\'{\i}n, Roberto and Stone, Peter},
  title         = {Deep reinforcement learning for robotics: a survey of real-world successes},
  year          = {2025},
  isbn          = {978-1-57735-897-8},
  publisher     = {AAAI Press},
  doi           = {10.1609/aaai.v39i27.35095},
  booktitle     = {Proceedings of the Thirty-Ninth AAAI Conference on Artificial Intelligence and Thirty-Seventh Conference on Innovative Applications of Artificial Intelligence and Fifteenth Symposium on Educational Advances in Artificial Intelligence},
  numpages      = {5},
  url           = {https://arxiv.org/abs/2408.03539}
}

@misc{Gazi2026,
  title         = {Statistical Reinforcement Learning in the Real World: A Survey of Challenges and Future Directions},
  author        = {Asim H. Gazi and Yongyi Guo and Daiqi Gao and Ziping Xu and Kelly W. Zhang and Susan A. Murphy},
  year          = {2026},
  url           = {https://arxiv.org/abs/2601.15353}
}

@misc{UngerSchaefer2026,
  title         = {Code Repository - {Chebyshev} Policies and the {Mountain Car} Problem: Reinforcement Learning for Low-Dimensional Control Tasks},
  author        = {Huber, Stefan and Unger, Hannes and Sch{\"a}fer, Georg and Rehrl, Jakob},
  url           = {https://github.com/JRC-ISIA/paper-2026-chebyshev-policies-low-dimensional-control-tasks},
  urldate       = {18.05.2026},
  year          = 2026
}
\bibliographystyle{icml2026}

\appendix
\onecolumn
\section{Details of the Analytical Solution to Mountain Car}
\subsection{Additional Proofs}

\subsubsection{\Cref{lem:oscperiod}}

We investigate the oscillation period depending on the start position $x_0 =
x(0)$ at rest, i.e, $\dot x(0) = 0$, when no action is applied. We recall that
$x$ fulfills the differential equation
\begin{align}
  \ddot x = -U(x) = -g \cos(3x).
\end{align}
By the symmetry of $U$ we have that $U(x_0) = U(\hat{x}_0)$, where $\hat{x}_0$
denoted the reflection of $x_0$ at $\hat{x}$ and $\hat{x}$ denoted the minimum
of $U$. That is, with no action the car would oscillate in the interval $[x_0,
\hat{x}_0]$.

We bring the differential equation into a standard form by the substitution
$\varphi = 3x + \frac{\pi}{2}$, which yields
\begin{align}
  \ddot \varphi = -3g \sin(\varphi).
\end{align}

Let us define $\alpha = 3x_0 + \frac{\pi}{2}$. Then $\varphi$ oscillates in the
interval $[-\alpha, +\alpha]$.
The period $T$ as in \cref{thm:osc} is then known to yield the elliptic
integral of the first kind, see p.~278 in \cite{Koenigsberger2000}. More
precisely,
\begin{align}
  T =  \frac{4}{\sqrt{3g}} K(k),
\end{align}
where $K$ is the complete elliptic integral of the first kind and $k =
\sin(\frac{\alpha}{2})$.

\subsubsection{\Cref{lem:csvariant}}
\label{sec:apx-proof-csvariant}

This lemma is a variant of Cauchy-Schwarz, which says that $\langle f,g \rangle
\le\|f\| \|g\|$ for elements $f, g$ of a inner-product vector space, and
$\langle f,g \rangle = \|f\| \|g\|$ iff $f$ and $g$ are positive multiple of
each other, i.e., the cosine between $f$ and $g$ being one. So minimizing
$\|f\|$ given $\langle f, \frac{g}{\|g\|} \rangle \le \| f\|$ while $\langle f,g
\rangle=1$ implies $f = g /\|g\|^2$.

A geometric argument is given as follows: The set $\{ f \colon \langle f, g
\rangle = 1 \}$ describes the hyperplane orthogonal to and supported by $g /
\|g\|^2$. Minimizing $\|f\|$ over this hyperplane is equivalent of finding the
orthogonal projection of the origin on this plane, i.e., the point $g /
\|g\|^2$.

\subsubsection{\Cref{thm:optalpha}}
\label{sec:apx-proof-optalpha}

    Assume we have some $\alpha$ that reaches the goal fulfills
    \eqref{eq:goalxicondition} to strict inequality. Then we could simply scale
    $\alpha$ down by a factor, reducing $\ell$, while still reaching the goal, but
    with less excessive kinetic energy. Hence, we can restrict $\alpha$ to
    solve \eqref{eq:goalxicondition} to equality. We set
    \begin{align*}
        f(\xi) = \tilde\alpha(\xi) / \sqrt[4]{-2U(\xi)} \quad\text{and}\quad
        g(\xi) = \amax \sqrt[4]{-2U(\xi)} / \left(\frac{1}{2}v_*^2 + U_g(\xi_*)\right).
    \end{align*}
    Note that $\ell = \|f\|^2$ and minimizing $\ell$ is equivalent to
    minimizing $\|f\|$. Further note that by equality in
    \eqref{eq:goalxicondition} we have $\langle f, g \rangle = 1$.
    %
    From \Cref{lem:csvariant} we conclude $f = g / \|g\|^2$, i.e.,
        $\tilde \alpha(\xi) = C \cdot \sqrt{-2U(\xi)}$
    and therefore
        $\alpha(t) = C \cdot \dot x(t)$
    for some constant $C$ as $\alpha \cdot \dot x \ge 0$.

\subsection{Details on Two-Phase Trajectories}
\label{sec:apx-twophasetrajectories}

Let us discuss the different cases as introduced in
\cref{sec:analysis-twophase} in further detail. We first ignore the constraint
$t_* \le \tmax$. The two phases are formed as follows:

\begin{enumerate}

    \item Phase~1 consists of the first $k-1$ strokes. An excessive velocity
      $\dot x$ when reaching the left wall is eliminated by the inelastic
      bump. Hence, the minimal loss is obtained by choosing $C_{1, k}$ just as
      small such that we reach $\xmin$ at zero velocity while maintaining $k-1$
      strokes in phase~1.

    \item In phase~2 we look for the smallest $C_{2, k}$ such that we reach
      $x_*$ at velocity $v_*$ from state $(\xmin, 0)$ in a single stroke. That
      is, the trajectory in phase~2 is actually independent of $k$ (still
      assuming $t_* \le \tmax$).

\end{enumerate}

Let us denote by $\vwall$ the velocity at which we hit the left wall at
$\xmin$. So far we discussed $\vwall = 0$.
The larger $k$ becomes the larger becomes $t_*$. At some $k$ we would violate
$t_* \le \tmax$. Roughly speaking, if we increase $C_{1, k}$ and $C_{2,k}$ then
we might reach $x_*$ at $t_* \le \tmax$ again, but $\vwall \neq 0$ when we
increase $C_{1,k}$.
However, note that a given $\vwall$ determines the $\ell$-optimal $C_{1,k},
C_{2,k}$ in order to reach $t_* = \tmax$ in this case: $C_{1,k}$ is chosen
small enough to reach the left wall with velocity $\vwall$. This determines the time
$\tau$ at which the wall is hit. Then we choose $C_{2,k}$ small enough to reach
$x_*$ in time $\tmax - \tau$ for phase~2. Hence, when $k$ is large enough such
that reaching $x_*$ at velocity $v_*$ would violate $t_* \le \tmax$, we have to
search for the $\ell$-optimal $\vwall$.

To sum up, for each $k$ we can determine the $\ell$-optimal $C_{1,k}$ and
$C_{2,k}$, yielding the optimal $\ell_k$. Then we exhaustively test all $k$ to
find the optimum.

\subsection{Optimal Policy for A-Priori Known Starting Location}

When the start position $x_0$ is given a priori, this knowledge can be
exploited to improve the optimal control strategy $\piana$ further; we denote
the resulting strategy by $\pioptxzero$.
\Cref{fig:policyheatmap_analytic_trajectory} shows $\pioptxzero$ with $x_0=-0.55$. We can
see that, contrary to $\piana$ depicted in \cref{fig:policycomparison}, its
trajectory reaches the left wall with a velocity close to zero, preventing loss
of excessive velocity due to the inelastic bump. We see that the bootstrapping
action $\alpha_\text{boot}(x, v) = 0.1$ for $|x - \hat{x}| \le 0.01$ extends to
both sides of the plane from the center, showing small impact given
that only a little fraction of the state space is affected.

\begin{figure}[h]
  \centering
  \includegraphics[width=0.7\textwidth]{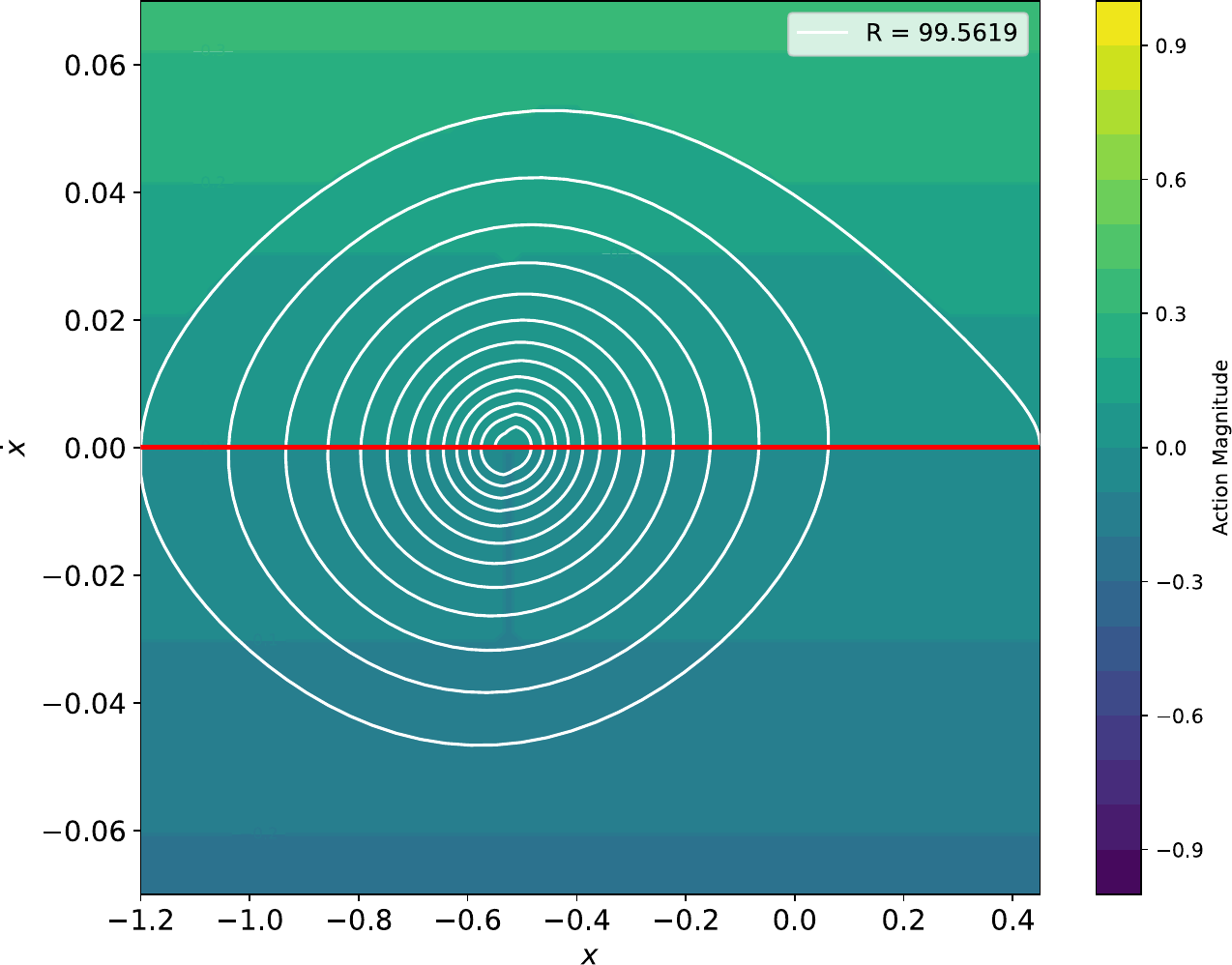}
  \caption{Optimal policy $\pioptxzero$ for $x_0=-0.55$ with $C_{1,k} = 3.2986$ and
  $C_{2,k} = 4.8358$, with $k=25$.}
  \label{fig:policyheatmap_analytic_trajectory}
\end{figure}

\subsection{Unconstrained Experiments}
\label{sec:apx-unconstrainedexp}

Recall our discussion after \cref{thm:optalpha}: When we have a small $C$ then
little action is applied, so for each stroke the potential $U$ is slowly
lowered and the number $k$ of strokes will be high until the goal is reached.
When $C$ is increased the number $k$ is reduced in discrete steps and so is
$\xi_* - \xi_0$.

While $k$ stays constant, however, increasing $C$ slightly increases $\xi_*
- \xi$ as each stroke becomes slightly longer. See more details in
\cref{sec:apx-unconstrainedexp}.

We move the left wall to $x_{\min} = -\frac{2 \uppi}{6}-0.45$ to verify
single-phase trajectories for rising values of the parameter $C$.
\cref{fig:single_phase_c} shows the instance $x_0 = -0.55$. As we increase $C$,
$x_*$ is reached at increased speed $\dot x(t_*)$, increasing $\ell$. Jumps
indicate a change in the stroke count $k$. If $\min_t x(t) = x_{k-1}$ becomes
less than $\xmin$ after a sufficiently large $k$ then these single-phase
solutions are not feasible.

Conversely, the position $x_{\min}$ of the left wall is feasible if and
only if the car can reach the goal in one stroke at $\alpha=1$. If the left wall is
approximately in the interval $[-0.651, 0.310]$ then this problem is infeasible,
and otherwise the Mountain Car can indeed reach the goal (from any start
position).

\begin{figure}[h]
  \centering
  \includegraphics[width=\textwidth]{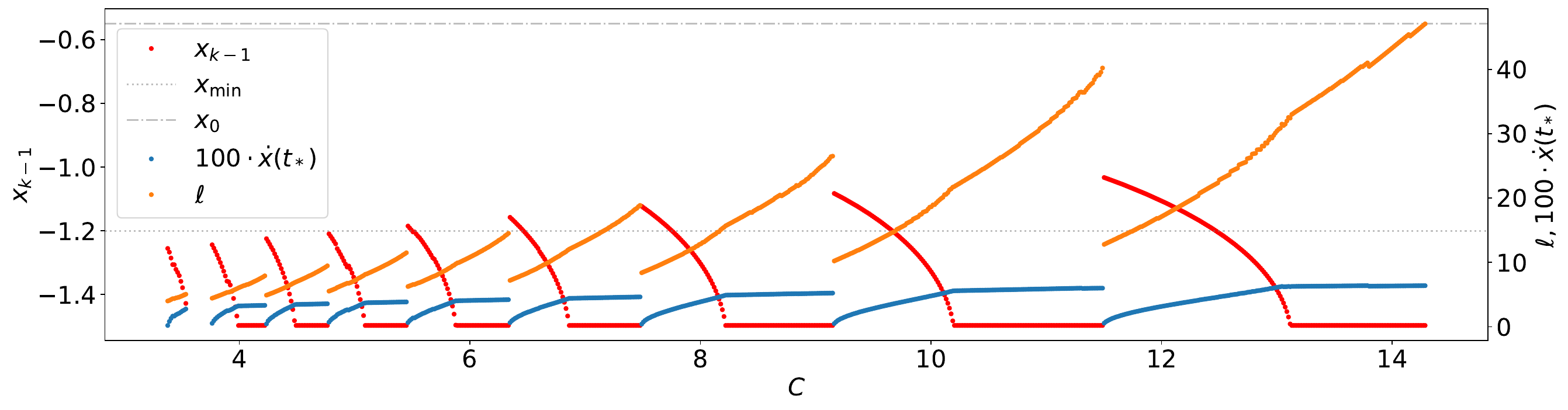}
  \caption{Loss $\ell$ over $C$ for $x_0 = -0.55$, together with $x_{k-1}$ and
  $\dot x(t_*)$ in the unconstrained setting.}
  \label{fig:single_phase_c}
\end{figure}

\subsection{Optimality of the Discrete Control Problem}
\label{sec:apx-optimality-discrete-problem}

To confirm that the continuous-time analytical policy is also the optimal
solution of the discrete-time case, the optimization problem
\begin{mini}
  {\alpha_i,i^*}{\sum_{i=0}^{i^*-1} \alpha_i^2}{}{}
  \addConstraint{x_{i^*} \geq 0.45} \addConstraint{|v_i| \leq 0.07}
  \addConstraint{x_i\geq-1.2} \addConstraint{ |\alpha_i|\leq 1}
  \addConstraint{ i^*\leq 999}
  \label{eq:discopt}
\end{mini}
was solved for the three starting points $x_0=-0.6$, $x_0=-\uppi/6$
and $x_0=-0.4$. The obtained return values only marginally differ in
the third decimal compared to the two-phase analytic solution. Since
the numerically obtained action sequence could be related to  a
local minimum of the objective -- particularly due to the large
number of optimization
variables, the guarantee that a global minimum was found is not given -- ,
another test was executed. The optimization problem was simplified by
only considering a single stroke, which reduces the number of
optimization variables significantly. Starting from $x_0=-0.6$ and
$x_0=-\uppi/6-0.001$, the policy from \cref{thm:optalpha} was applied to
the discrete-time simulation model with an arbitrarily chosen $C=20$.
Using this policy, the car will execute one stroke until it
reaches zero velocity at position $x_e$. This position is then used
as target position $x_i^*$ in the optimization problem~\eqref{eq:discopt}.
This time, the initial guess of the action sequence is zero, i.e., the
solver is not provided with any solution and now it should still find
the same action sequence as given by \cref{thm:optalpha}.

Indeed, the optimizer could slightly improve the loss compared to the
given policy ($x_0=-0.6$: $0.5835$ analytic vs. $0.57246$ optimized;
$x_0=-\uppi/6-0.001$: $9.983\cdot 10^{-5}$ analytic vs. $9.772\cdot 10^{-5}$
optimized). However, the shape of the action sequence is very similar to
the analytically derived policy. 
To get an idea of the effect of discretization, in addition to the
numerical approximation of the
integral of the squared action over time, equation~\eqref{eq:lossspace}
was numerically evaluated, i.e., the loss was computed by integration
along the position. The observed discrepancy between the two
approximations of the integral was larger than the discrepancies
observed between analytical solution of the continuous time
problem compared to the numerical solution of the discrete-time problem.
Therefore, we conclude that the analytically derived policy for
the continuous time case also holds for the discrete-time case,
since the observed differences are smaller than effects originating
from the time discretization.

\section{Details on Chebyshev Policies}
\subsection{Multi-Variate Horner Scheme}

Let $p(x) = a_0 + a_1 x + \cdots + a_d x^d$ denote a polynomial of degree $d$
then the classic Horner scheme factors the $x$ as follows:
\begin{align}
  \label{eq:horner}
  p(x) = \underbrace{\left(\cdots \left((a_d \cdot x + a_{d-1})\cdot x + a_{d-2}\right) \cdots + a_1\right)\cdot x}_{\text{$d$ parentheses}} + a_0
\end{align}
The number of multiplications and additions used is $d$, which is the smallest
possible. There is varying literature on multi-variate generalizations. Since
we use the notion of max-degree in this paper, we briefly provide an analysis
for the number of multiplications and additions for $n$-variate max-degree $d$
polynomials here.

We see an $n$-variate polynomial as a polynomial of polynomials of polynomials
and so on, each nesting level corresponding to one variable. Hence, we apply
the Horner scheme per variable, which is per nesting level. That is, in
\cref{eq:horner}, when $p$ is an $n$-variate polynomial of max-degree $d$ then
$a_i$ are $(n-1)$-variate polynomials, also of max-degree $d$.

\begin{lemma}
  Evaluating an $n$-variate max-degree $d$ polynomial with the multi-variate
  Horner scheme leads to $(d+1)^n - 1$ multiplications and additions.
\end{lemma}

\begin{proof}[Proof by induction]
  The case $n=1$ is given in \cref{eq:horner}. For the induction step $n \to
  n+1$ observe that the $a_0, \dots, a_d$ in \cref{eq:horner} each take by
  assumption $(d+1)^n - 1$ multiplications and additions, so we have in total
  $((d+1)^n - 1) \cdot (d+1) + d = (d+1)^{n+1} - 1$ multiplications and
  additions.
\end{proof}

Note that such a polynomial has $(d+1)^n$ monomials.

\section{Details on Mountain Car Experiments}
\label{sec:details-eval-mountaincar}

\subsection{Training and Evaluation Protocol}
\label{sec:details-training-eval-mountaincar}

REINFORCE training was conducted with AdamW optimizer and a discount factor of
$0.9$, using the following reproducible protocol: We trained $20$ policies for
$100$ episodes. In each episode the environment picks $x_0 \sim
\mathcal{U}([-0.6, -0.4])$. Then $\sigma(s)$ in the stochastic policy was set
zero (for a deterministic evaluation) and $50$ evaluation episodes were
conducted. The one policy with the highest average return is picked for
evaluation.

We did this for Chebyshev policies of degree 3. (Degree 3 performs slightly
better than higher degrees, e.g., degree 5.) We call the resulting policy
\textsc{CH-3-REI}. In \cref{sec:ch3evaluation} we give further details using
other optimizers than AdamW. With \textsc{CH-3-ARS} and \textsc{CH-3-PPO} we
also trained $20$ policies each, with $80$~k steps for the former and $70$~k
steps for the latter. We then took pretrained RL Baselines3 Zoo agents on the
MountainCarContinuous-v0 problem, see \cite{Raffin2024}, namely for \gls{ars}, \gls{sac},
\gls{ppo}, \gls{ddpg}, \gls{td3}, \gls{a2c}, and \gls{trpo}.

Each of the agents together with $\piana$ were evaluated as follows: We ran them
with $x_0$ chosen at $100$ evenly spaced points in $[-0.6, -0.4]$ and recorded
the achieved return $R$. The mean return $\overline{R}$ is a faithful estimator
for the expected return $\mathbb{E}_{x_0 \sim \mathcal{U}([-0.6, -0.4])}(R)$.
The three best performing \gls{sota} agents turned out to be \gls{ars}, \gls{sac} and \gls{ppo},
as summarized in \cref{tab:evalpolicies}. A full comparison is given in
\cref{sec:apx-fullcomp-policies-mountaincar}.
The regret $r$ is the difference of $\overline{R}$ of the evaluated policy
compared to the one of $\piana$. In particular, \textsc{CH-3-REI} trained by
classical REINFORCE improves the regret ($0.77$) compared to \gls{ars} ($2.72$), \gls{sac}
($4.78$) and \gls{ppo} ($5.48$) by a factor of $3.5$, $6.2$ and $7.1$, respectively,
which is surprising given that REINFORCE as the first policy-gradient algorithm
is considered clearly inferior to modern algorithms like \gls{ppo} or \gls{sac}.

All experiments were conducted on an Intel Core i7-7800X CPU with $32$~GB of
RAM. The most time consuming task was the training protocol for
\textsc{CH-3-REI} which took \SI{90}{\minute} with the evaluation taking
another \SI{30}{\minute}. Less time is required for training and evaluation of
\textsc{CH-ARS} and \textsc{CH-PPO} as well as evaluating the RL Baselines3 Zoo
agents.

\subsection{Training Details for REINFORCE with Different Optimizers}
\label{sec:ch3evaluation}

Given the simple nature of REINFORCE as the first policy-gradient \gls{rl}
algorithm, we took a deeper look on the effect of different optimizers on the
training of Chebyshev policies. (This is partly motivated by investigations of
\cite{Waclawek2024} on uni-variate piecewise Chebyshev polynomials in
supervised learning tasks, where the effect of different optimizers was
pronounced.)

We reimplemented the classical Monte Carlo policy gradient REINFORCE algorithm
according to \cite{Sutton2018}, as it is not part of PyTorch.
We trained Chebyshev policies of max-degree $3$ using the classical REINFORCE
algorithm according to \cite{Sutton2018}, utilizing gradient-based PyTorch
optimizers. As in \cref{sec:details-training-eval-mountaincar} $20$ policies
were trained for $100$ episodes, followed by $50$ evaluation episodes. The
results of the evaluation experiment for \textsc{CH-3-REI} are depicted in
\cref{fig:ch3_training_evaluation}. As it is typical for REINFORCE, we have
a large spread in performance of the trained policies.

The one policy with the highest average return was selected and is denoted by
\textsc{CH-3-REI}, which stems from the AdamW optimizer. The step size for
training was set to $0.0003$ and $\theta^{(\sigma)}_{i_1, \dots, i_n}$ where
initialized s.t. $\sigma(s)$ is a value close to $1$, enabling initial
exploration. $\sigma(s)=0$ during evaluation. Training of policies using SGD or
L-BFGS optimizers diverged.

\begin{figure}[h]
  \centering
  \includegraphics[width=\textwidth]{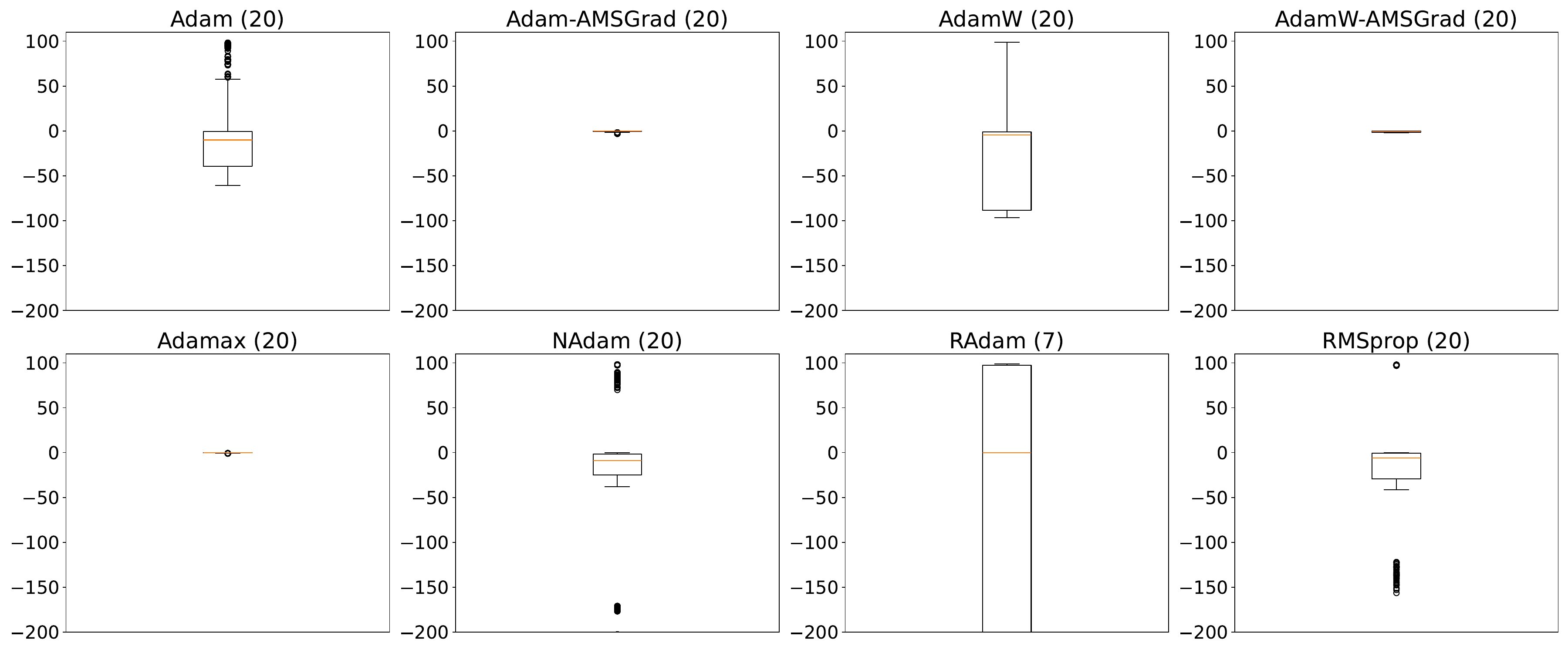}
  \caption{Mountain Car: Evaluation of \textsc{CH-3-REI} results trained with different optimizers, number
  of non-diverging policies utilized is shown in brackets ($20$ before
  training). $50$ episodes per policy, each datapoint is the return of one
  episode.}
  \label{fig:ch3_training_evaluation}
\end{figure}

\subsection{PPO with Different MLP Architectures}
\label{sec:apx-ppo-on-mlp-mountaincar}

A pressing research question from our analysis in \cref{sec:regret-on-sota}
concerns the surprisingly high regret of \gls{mlp} policies with \gls{sota}
\gls{rl} algorithms, like \gls{sac}, \gls{ppo} or \gls{ars} alike.

Let us focus on \gls{ppo}, which is known for its wide applicability, and ask whether
the default network size $[64, 64]$ of the \gls{mlp} policy would explain the
impaired performance. We can exclude underfitting from the simplicity of the
optimal policy $\piana$. Hence, we can focus on a possible overfitting,
although $[64, 64]$ is not quite large.

In \cref{tab:evaluation_mlp_sizes} we reported on the achieved results for
reduced number of layers and reduced sizes of layers. To sum up, when the
network size is reduced, the agent's performance becomes even more degraded.
The results of the default setting match the one reported on RL Baselines3 Zoo.

\begin{table}[hh]
  \centering
  \caption{Performance on Mountain Car. Mean, min and max return for \gls{ppo} agents using different \gls{mlp}
  architectures. Evaluated on $20$ evenly spaced starting positions $x_0 \in
  [-0.6, -0.4]$.}
  \label{tab:evaluation_mlp_sizes}
  \begin{tabular}{lcccr}
    \toprule
    Layer Dimensions     & $\overline{R}$ & $R_{\text{min}}$ & $R_\text{max}$\\
    \midrule
    $[64,64]$ (original) & $94.22$        & $91.16$          & $95.15$\\
    $[32,32]$            & $92.60$        & $89.02$          & $95.82$\\
    $[64]$               & $90.97$        & $90.03$          & $95.13$\\
    $[16,16]$            & $-0.22$        & $-0.32$          & $-0.16$\\
    $[16]$               & $-0.55$        & $-0.55$          & $-0.53$\\
    $[32]$               & $-5.34$        & $-7.80$          & $-4.02$\\
    \bottomrule
  \end{tabular}
\end{table}

\subsection{REINFORCE Does Not Succeed with MLP Policies}
\label{sec:apx-reinforce-fails-mlp-mountaincar}

A surprising result of \cref{sec:eval-cheby-beats-mlp} was that
\textsc{CH-3-REI} based on REINFORCE performs significantly better than all
\gls{mlp} policies trained by different \gls{rl} algorithms, while REINFORCE
would not be able to actually train a successful \gls{mlp} policy. Here we
report on the details on the latter claim.


We again repeat the same training protocol as in
\cref{sec:details-training-eval-mountaincar}: Recall that a stochastic policy
$\pi$ maps a state $s$ to a pair $(\mu(s), \sigma(s))$, which are here modeled
by a \gls{mlp}, respectively. The $\mu$-net is initialized by default
initialization of PyTorch. The $\sigma$-net is initialized with small random
numbers, only the last layer receives a bias to $0.25$. That is, $\sigma$ as
a map initially gives a constant close to $1$ over the state space with small
random fluctuations. (This identical to what we did for Chebyshev policies.)

We tested various \gls{mlp} architectures, starting with the default 2-layer
architecture $[64, 64]$, but also smaller single- and two-layer variants.
(Recall that the optimal policy $\piana$ is quite simple, so $[64, 64]$ should
by far suffice in terms of model capacity.)

In \cref{fig:mlp_reinforce_training_evaluation}, we report on the results for
the six \gls{mlp} architectures and eight different optimizers, and for each
combination we trained 20 agents, resulting in a total of 960 training
attempts. None of these solved the Mountain Car task, only one agent (RAdam,
network size $[8]$) at times reaches the goal with a mean return of about 60.

\begin{figure}[h]
  \centering
  \includegraphics[width=\textwidth]{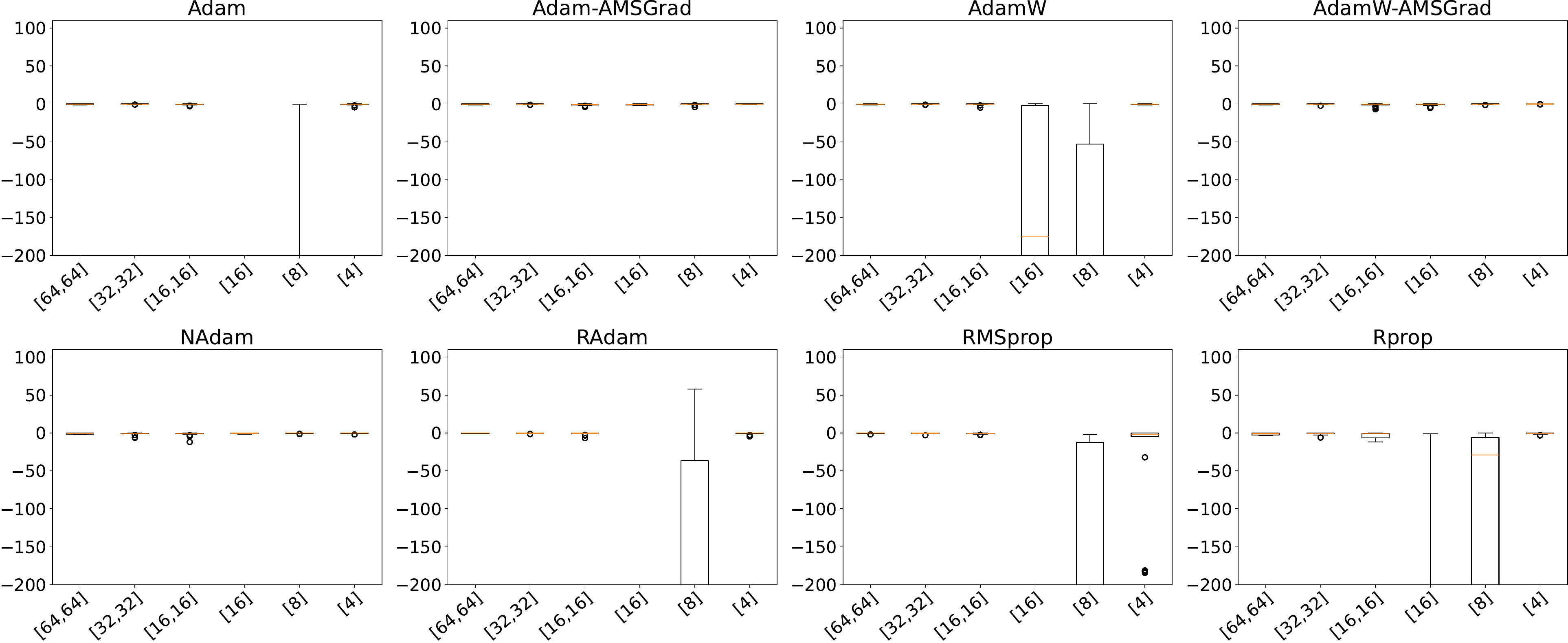}
  \caption{Mountain Car: Evaluation of \gls{mlp} REINFORCE results trained with different
  optimizers. $50$ episodes per policy, each datapoint is the return of one
  episode.}
  \label{fig:mlp_reinforce_training_evaluation}
\end{figure}

\subsection{Full Comparison of Chebyshev Policies Against All RL Baseline3 Zoo Agents}
\label{sec:apx-fullcomp-policies-mountaincar}

In the following take a deeper look on the results briefly summarized in
\cref{tab:evalpolicies} and present a full version of this table in
\cref{tab:evaluation_x0_range}. We compare all pretrained RL Baselines3 Zoo
agents, which are in decreasing order of performance: \gls{ars}, \gls{sac},
\gls{ppo}, \gls{ddpg}, \gls{td3}, \gls{trpo}, \gls{a2c} and \gls{tqc}. In
\cref{tab:evaluation_x0_range} we leave out the \gls{tqc} agent as its
performance lies behind substantially ($\overline{R} = 61.14$ for $x_0 \in
[-0.6, -0.4]$). In the following discussion a couple of aspects of the presented
results.

\begin{table}[p]
  \centering
  \caption{Performance on Mountain Car. Mean, standard deviation, min, and max
    return $R$, and the same for velocity $v_*$ at target and episode length
    $t_*$ for different trained agents over $100$ evenly spaced $x_0 \in [-0.6,
  -0.4]$.}
  \label{tab:evaluation_x0_range}

  \scriptsize
  \begin{tabular}{lcccccccccccccc}
    \toprule
    Policy
                      & $\overline{R}$    & $\std(R)$   & $R_{\min}$        & $R_{\max}$
                      & $\overline{v_*}$  & $\std(v_*)$ & $v_{*,\min}$      & $v_{*,\max}$
                      & $\overline{t_*}$  & $\std(t_*)$ & $t_{*,\min}$      & $t_{*,\max}$
                      & $\norm{\cdot}_2$\\
    \midrule

    $\pioptxzero$     & $99.59$           & $0.0463$                  & $99.47$           & $99.68$
                      & $4.7\cdot10^{-4}$ & $0.0000$                  & $4.7\cdot10^{-4}$ & $4.7\cdot10^{-4}$
                      & $955.14$          & $30.33$                   & $849$             & $999$ &  \\

    $\piana$          & $99.39$           & $0.0768$                  & $99.15$           & $99.52$
                      & $4.7\cdot10^{-4}$ & $0.0000$                  & $4.7\cdot10^{-4}$ & $4.7\cdot10^{-4}$
                      & $769.25$          & $90.8347$                 & $571$             & $968$ &  \\ \midrule

    \textsc{CH-3-ARS} & $98.95$           & $0.0987$                  & $98.74$           & $99.11$
                      & $1.8\cdot10^{-2}$ & $0.0032$                  & $7.2\cdot10^{-3}$ & $2.0\cdot10^{-2}$
                      & $471.65$          & $125.6362$                & $325$             & $985$ & $0.1523$\\

    \textsc{CH-3-REI} & $98.62$           & $0.1661$                  & $98.31$           & $98.89$
                      & $2.4\cdot10^{-2}$ & $0.0069$                  & $3.0\cdot10^{-3}$ & $2.8\cdot10^{-2}$
                      & $396.46$          & $109.5182$                & $246$             & $838$ & $0.0680$ \\

    \textsc{CH-3-PPO} & $98.10$           & $0.2217$                  & $97.61$           & $98.42$
                      & $2.3\cdot10^{-2}$ & $0.0053$                  & $4.8\cdot10^{-3}$ & $2.6\cdot10^{-2}$
                      & $469.97$          & $128.3894$                & $334$             & $985$ & $0.0865$\\ \midrule

    ARS               & $96.67$           & $0.8562$                  & $92.51$           & $97.42$
                      & $4.2\cdot10^{-2}$ & $0.0130$                  & $1.1\cdot10^{-2}$ & $5.3\cdot10^{-2}$
                      & $239.28$          & $84.1711$                 & $152$             & $610$ & $0.2105$ \\

    SAC               & $94.61$           & $1.2345$                  & $89.70$           & $95.77$
                      & $3.8\cdot10^{-2}$ & $0.0150$                  & $1.2\cdot10^{-2}$ & $6.1\cdot10^{-2}$
                      & $105.77$          & $33.1034$                 & $76$              & $179$ & $0.3173$ \\

    PPO               & $93.91$           & $1.1693$                  & $90.86$           & $95.23$
                      & $3.4\cdot10^{-2}$ & $0.0066$                  & $5.7\cdot10^{-3}$ & $3.7\cdot10^{-2}$
                      & $298.01$          & $93.6045$                 & $202$             & $858$ & $0.2730$ \\

    DDPG              & $93.51$           & $0.0486$                  & $93.43$           & $93.56$
                      & $3.9\cdot10^{-2}$ & $0.0048$                  & $2.9\cdot10^{-2}$ & $4.6\cdot10^{-2}$
                      & $66.37$           & $0.4828$                  & $66$              & $67$  & $0.4267$ \\

    TD3               & $93.48$           & $0.0772$                  & $93.36$           & $93.62$
                      & $3.4\cdot10^{-2}$ & $0.0022$                  & $2.9\cdot10^{-2}$ & $3.8\cdot10^{-2}$
                      & $65.95$           & $0.7794$                  & $65$              & $67$  & $0.4260$ \\

    TRPO              & $92.49$           & $0.3872$                  & $90.16$           & $92.78$
                      & $5.2\cdot10^{-2}$ & $0.0078$                  & $3.7\cdot10^{-2}$ & $6.3\cdot10^{-2}$
                      & $86.84$           & $9.5433$                  & $81$              & $157$ & $0.4008$ \\

    A2C               & $91.16$           & $0.2555$                  & $90.46$           & $91.60$
                      & $4.8\cdot10^{-2}$ & $0.0085$                  & $3.2\cdot10^{-2}$ & $6.0\cdot10^{-2}$
                      & $90.44$           & $2.6583$                  & $86$              & $98$  & $0.4194$ \\

    \bottomrule
  \end{tabular}
\end{table}

\begin{figure}[p]
  \centering
  \includegraphics[height=55mm,trim={0mm 3.5mm 0mm 2.5mm},clip]{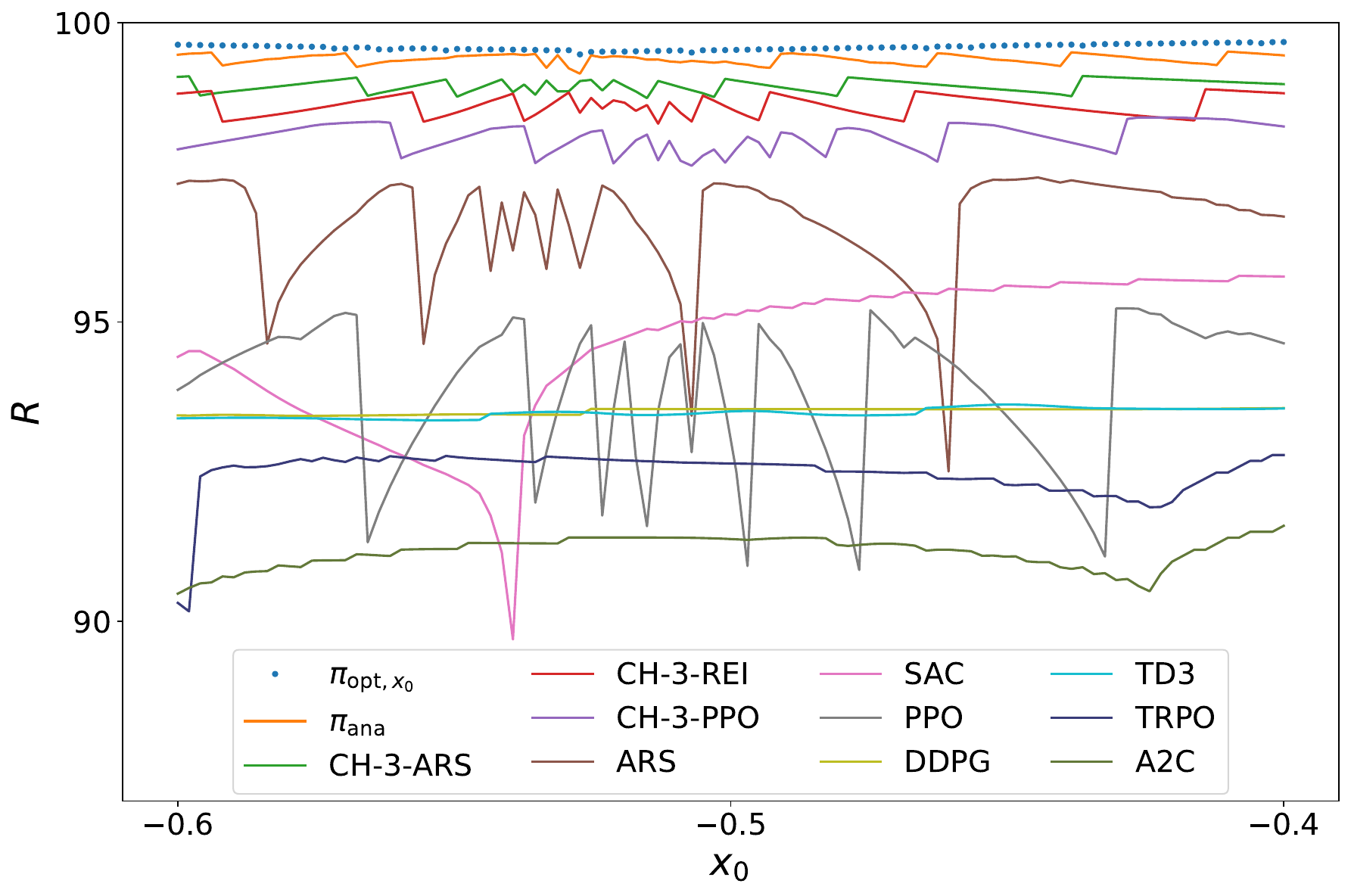}
  \caption{The return $R$ over start positions $x_0$ for all policies from \cref{tab:evaluation_x0_range}.}
  \label{fig:apx-mc-return-over-x0}
\end{figure}

\begin{figure}[p]
  \centering
  \includegraphics[height=92mm,trim={0mm 5mm 0mm 3mm},clip]{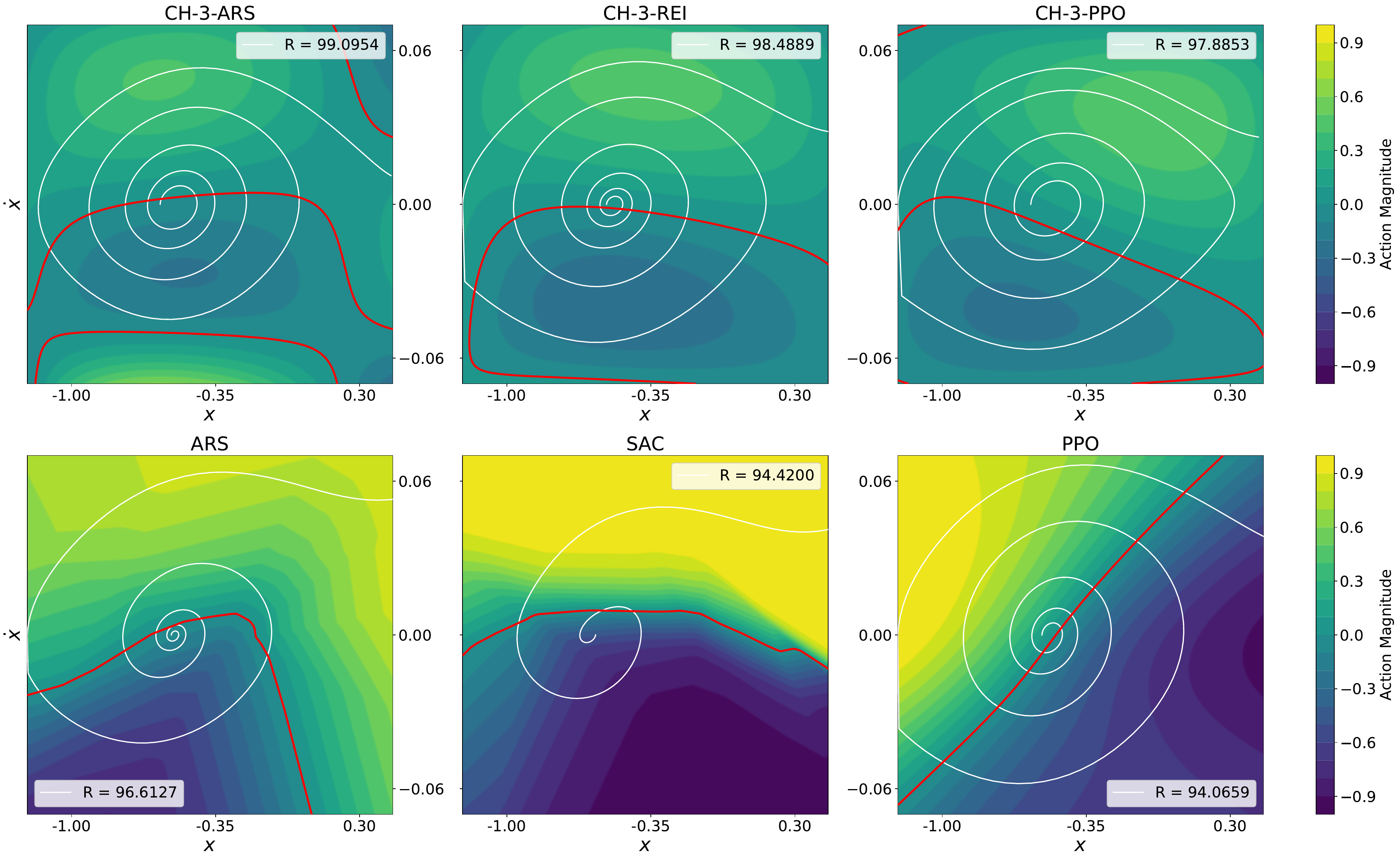}
  \caption{The figures plot the actions of Chebyshev agents and the three
  top-performing neural agents over the state space, the zero-actions in red and
  in white a trajectory from $x_0 = -0.55$.}
  \label{fig:policycomparisonext}
\end{figure}

First of all, we see that the mean target velocity $\overline{v_*}$ of the
optimal policy is essentially zero, independent of $x_0$. This is what we
expect from optimality, otherwise we would suffer from excess kinetic energy at
the goal. Furthermore, $\pioptxzero$ essentially exploits the maximum possible
time $t_*$ to reach the goal, with the remainder to 999 being explained by the
restriction that the number of strokes is simply an integral number. In
contrast to $\pioptxzero$, the analytical optimal policy $\piana$ does not have
the advantage of known $x_0$ a priori and hence has lower $\overline{t_*}$.

As already discussed in \cref{tab:evalpolicies}, all Chebyshev policies
outperform all the \gls{mlp} policies. Moreover, when we look at the standard
deviation, we observe that the Chebyshev policies have lower standard deviation
than \gls{mlp} policies, except the low-performing \gls{mlp} policies \gls{ddpg} or
worse.

Taking a closer look at the minimum return $R_{\min}$, we see that the top
\gls{mlp} policy \gls{ars} delivers a very inconsistent performance, at times falling
below low-performing \gls{mlp} policies, like \gls{ddpg} or \gls{td3}. We highlight this
behavior in \cref{fig:apx-mc-return-over-x0}, where we plot the return $R$ over
all start positions $x_0$.

When we look at the mean time $\overline{t_*}$ when a policy reaches the goal,
we consistently see that the Chebyshev policies exploit $t_* \le 999$ more
effectively an come closer to $\piana$. What we also observe is that the
low-performing policies have a low standard deviation $\std(t_*)$, especially
for \gls{ddpg} to \gls{a2c}. This indicates that they learn a policy leading to
trajectories of equal length, quite independent of the start position $x_0$,
which is consistent with low values of $\overline{t}_*$ in total, i.e., they
reach the goal without exploiting multiple forth-and-back oscillations. This
observation is confirmed by the plotted trajectory for \gls{sac} in
\cref{fig:policycomparisonext}.

Furthermore, when we compare a policy
$\pi$ to $\piana$ in terms of the $L_2$-norm, i.e., by looking at $\|\pi
- \piana\|_2$ over the state space domain $[-1.2, 0.45] \times [-0.07, 0.07]$,
  then again the Chebyshev policies approximate $\piana$ much closer than the
  \gls{mlp} policies. Taking a closer look at the \textsc{CH-3-ARS} in
  \cref{fig:policycomparisonext}, we see that the difference to $\piana$ mainly
  occurs in areas with no or less relevance for the occurring trajectories.

\subsection{In-Depth Comparison of Policies}
\label{sec:apx-indepth-policies-mountaincar}

In \cref{fig:policycomparisonext} we extend our discussion from
\cref{fig:policycomparison} by extending the policy plots to all Chebyshev
policies and the three top performing \gls{mlp} policies. The observations we
mentioned in \cref{sec:control-strategies-compare} generalize to the larger set
of policies in \cref{fig:policycomparisonext} as follows.

All \gls{mlp} policies generate larger actions than Chebyshev policies.
Consequently, the number of strokes for \gls{mlp} is smaller, so is $t_*$,
which lowers the return. Only \gls{ppo} displays a larger number of strokes, but only
because it counteracts the dynamics of Mountain Car, i.e., it outputs actions
in opposite direction of the velocity of the car.

This counter action of the dynamics is also present for \gls{ars}, though at regions
of the state space with less impact on the trajectories. This also illustrates
limitations of the quality criterion $\|\pi - \piana\|_2$ when the deviation is
mainly attributed to irrelevant regions of the state space.

\subsection{Curse of Dimensionality}
\label{sec:apx-dimensionality}

For a state space of dimension $n$ and max-degree $d$, we have $(d+1)^n$
Chebyshev polynomials. To investigate the impact of $d$ on the return,
computational performance, and numerical stability, we repeat training of
Chebyshev policies using \gls{ppo} on Mountain Car with increasing max-degree
from $1$ up to $50$. For each degree, again, $20$ agents were trained
independently and evaluated by averaging the return over $100$ deterministic
starting states sampled from the initialization range. The results are shown in
\cref{tab:evaluation_rising_degrees}.

Overall, return improves rapidly once the policy class becomes sufficiently
expressive. Moreover, the return remains stable over a wide range of higher
degrees, still beating all neural \gls{sota} agents even with $d=20$. Beginning
with degrees around $30$, however, several agents do not succeed in reaching
the goal flag any longer and standard deviation of overall return increases.

However, these results highly depend on the choice of \gls{rl} algorithm.
Classical REINFORCE would already fail at max-degree 10. Also, performance
(i.e., simulation steps per second) drops, though not quadratic (as the number
of polynomials). However, we have not optimized for computational performance
yet, e.g., employ dynamic programming strategies on the recursive computation
scheme of multi-variate Chebyshev polynomials.

\begin{table}[p]
  \centering
  \caption{\textsc{CH-3-PPO} performance on Mountain Car for different
  polynomial degrees. Training was conducted with $20$ agents for each degree,
  showing mean return $\overline{R}$ and standard deviation across all agents as well as the best
  performing agent per degree. 
  Evaluation conducted over $100$ evenly spaced $x_0 \in [-0.6, -0.4]$.}
  \label{tab:evaluation_rising_degrees}
  \begin{tabular}{lccccc}
    \toprule
    Degree & $\overline{R}$ all    & $\std(R)$ all   & $\overline{R}$ best    & $\std(R)$ best & Steps/s \\
    \midrule
    1 & -0.0975 & 0.0842 & -0.0043 & 0.0002 & 61.0 \\
    2 & -0.6159 & 0.1649 & -0.3784 & 0.2569 & 53.9 \\ 
    3 & 88.4433 & 30.117 & 98.9623 & 0.0769 & 50.4 \\ 
    4 & 97.9296 & 0.1271 & 98.1458 & 0.2278 & 46.9 \\ 
    5 & 97.78 & 0.1565 & 98.2048 & 0.1625 & 43.0 \\
    6 & 97.5848 & 0.1891 & 98.0292 & 0.1855 & 38.5 \\
    7 & 97.6631 & 0.2513 & 98.11 & 0.2196 & 36.9 \\
    8 & 97.5722 & 0.3585 & 98.1088 & 0.2061 & 34.5 \\
    9 & 97.623 & 0.244 & 97.9115 & 0.2099 & 32.0 \\
    10 & 97.3698 & 0.2612 & 97.7654 & 0.2519 & 29.2 \\
    11 & 97.4387 & 0.3251 & 97.838 & 0.2133 & 27.8 \\
    12 & 97.3455 & 0.2943 & 97.8038 & 0.2306 & 25.6 \\
    13 & 97.3748 & 0.2645 & 97.8012 & 0.1816 & 23.2 \\
    14 & 97.2876 & 0.2562 & 97.6679 & 0.1562 & 22.8 \\
    15 & 97.1798 & 0.3041 & 97.6372 & 0.3292 & 21.9 \\
    16 & 97.0765 & 0.5233 & 97.7306 & 0.1682 & 19.4 \\
    17 & 97.0275 & 0.3151 & 97.5145 & 0.1303 & 18.6 \\
    18 & 96.952 & 0.2965 & 97.4589 & 0.1814 & 18.0 \\
    19 & 96.6435 & 0.4085 & 97.3213 & 0.185 & 16.1 \\
    20 & 96.6991 & 0.7013 & 97.3904 & 0.1901 & 15.4 \\
    30 & 77.2577 & 38.7163 & 95.7942 & 0.331 & 9.8 \\
    40 & 80.4135 & 19.5194 & 94.6989 & 0.4354 & 6.0 \\
    50 & 80.7739 & 21.937 & 94.9583 & 0.5967 & 4.0 \\
    \bottomrule
  \end{tabular}
\end{table}

\subsection{Evaluation with Additional Stable Baselines3 RL Algorithms}
\label{sec:mc-ext-eval-sb3}

In the following, we extend our experimental evaluation to different \gls{rl}
algorithms available in Stable Baselines3. However, in order to avoid laborous
native implementation of Chebyshev policies for different \gls{rl} algorithms,
we followed the following approach: We use a single neuron without bias and
linear activation and pass it as network architecture to the Stable Baselines3
API, to effectively form the linear combination of the Chebyshev polynomials as
in \cref{eq:cheby-mu}.

We again run experiments on Mountain Car. To verify consistency, we also re-ran
\gls{ppo} and \gls{ars} experiments with this setup, and added \gls{sac} as
additional algorithm. In this way, results now cover an on-policy, off-policy
and random search algorithm. Results (20 independently trained agents,
evaluated over 100 uniform initial states) are shown in \cref{tab:mc-wrapper}.
We see that both \gls{ars} and \gls{ppo} results align with what we have seen
with native implementations (see \cref{tab:evaluation_x0_range}).
\textsc{CH-SAC} follows this line of results, still outperforming all \gls{mlp}
\gls{sota} agents (including \gls{sac} with \gls{mlp}).

\begin{table}[tb]
  \centering
  \caption{Performance on Mountain Car with Chebyshev representation layer with single-node \gls{mlp}
  architecture. Training 20 agents with distinct seeds and picking the result
  with the highest mean return $\overline{R}$. Showing mean return $\overline{R}$
  (regret $r$), standard deviation as well as min and max return.}
  \label{tab:mc-wrapper}
  \small
  \begin{tabular}{llccc}
    \toprule
    Pol.~$\pi$         & $\overline{R}\uparrow$ ($r\downarrow$) & $\std(R)$ & min $R$ -- max $R$ \\ \midrule
    $\piana$           & 99.39                                  & 0.0768 & 99.15   -- 99.52   \\ \midrule
    \textsc{\bfseries CH-3-ARS}  & 98.72 (0.67)                 & 0.13 & 98.45--98.93   \\
    \textsc{\bfseries CH-4-PPO}  & 98.47 (0.92)                 & 0.16 & 98.17--98.74   \\
    \textsc{\bfseries CH-3-SAC}  & 98.19 (1.20)                 & 0.22 & 97.64--98.52   \\
    \bottomrule
  \end{tabular}
\end{table}

\section{Experimental Results on Additional Environments}
\label{sec:apx-further-envs}

\subsection{Pendulum}

In this section we extend our discussion from \cref{sec:eval-cheby-other-tasks}
concerning the performance of Chebyshev policies against pretrained RL
Baselines3 Zoo agents on the Gymnasium environment \emph{Pendulum-v1}. For the
sake of comparison to other environments, we again trained Chebyshev policies
with \gls{ars}, \gls{ppo} and REINFORCE. Preliminary experiments revealed that \gls{ars} performs
best with max-degree 6 Chebyshev policies, \gls{ppo} performs best with max-degree
5 and REINFORCE performs best with max-degree 3, and they are accordingly named
\textsc{CH-6-ARS}, \textsc{CH-5-PPO}, and so on. In the following we exclude
\textsc{CH-3-REI} as it lags behind significantly with $\overline{R} = -466.29$.

\begin{table}[b] \centering
  \caption{Performance on pendulum environment. Evaluation over $50 \times 50$
  evenly spaced initial angles from $[-\pi, \pi]$ and angular velocities from
$[-1, 1]$, showing mean, standard deviation, min and max returns.}
  \label{tab:evalpoliciespendulum}
  \small
  \begin{tabular}{lccccc}
    \toprule
    Policy                      & $\overline{R}$ & $\std(R)$ & $\min R$ & $\max R$ \\
    \midrule
    \textsc{SAC}                & -147.24        & 83.39     & -377.46  & -1.08   \\
    \textsc{TQC}                & -147.83        & 84.22     & -380.56  & -1.01   \\
    \textsc{\bfseries CH-6-ARS} & -150.80        & 88.12     & -392.02  & -0.21   \\
    \textsc{TD3}                & -155.00        & 92.71     & -377.14  & -0.51   \\
    \textsc{DDPG}               & -156.64        & 106.50    & -1491.25 & -1.26   \\
    \textsc{\bfseries CH-5-PPO} & -162.75        & 98.62     & -486.60  & -1.97   \\
    \textsc{A2C}                & -167.90        & 102.09    & -528.22  & -0.08   \\
    \textsc{\bfseries PPO}      & -176.16        & 107.17    & -516.15  & -3.83   \\
    \textsc{TRPO}               & -180.65        & 141.49    & -1491.89 & -0.25   \\
    \textsc{\bfseries ARS}      & -218.32        & 159.30    & -784.63  & -0.07   \\
    \bottomrule
  \end{tabular}
\end{table}

We summarize our experimental results in \cref{tab:evalpoliciespendulum}. As we
discussed in \cref{sec:eval-cheby-other-tasks}, Chebyshev approximators improve the mean
return of \gls{mlp} policies for \gls{ars} and \gls{ppo}. We also list the standard
deviation $\std(R)$ in \cref{tab:evalpoliciespendulum} and see that the
attained returns are quite spread, which is expected given that a random
initialization of the pendulum in upright position enables a possible return of
zero, while the opposite initial position lowers the return substantially.
(The optimal control for this environment is unknown, and hence we can only
speculate on the regret of the current \gls{sota}.)

\begin{figure}[tb]
  \centering
  \includegraphics[width=.6\textwidth]{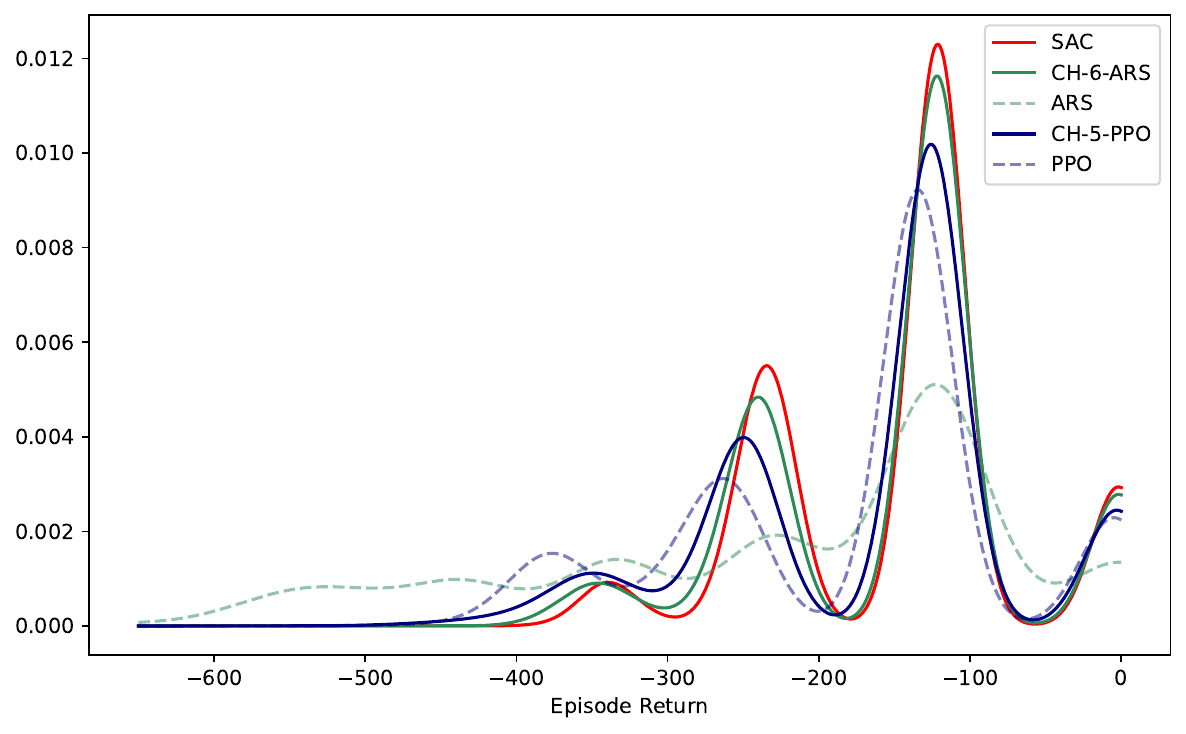}
  \caption{Pendulum: Density function of the return distribution for Chebyshev
  policies and their \gls{mlp} counterparts along with the top-performing policy
trained by \gls{sac}.}
  \label{fig:pendulum-return-density}
\end{figure}

The high standard deviation motivated us to take a closer look at the density
function of the return distribution, see \cref{fig:pendulum-return-density}.
Here we see that indeed the Chebyshev policies move density mass from lower
returns to higher returns and, moreover, the performance of the top-performer
\gls{sac} and \textsc{CH6-ARS} are very close, what we also see from
\cref{tab:evalpoliciespendulum}. Interestingly, for \gls{mlp} policies, \gls{ars}
actually performed worst.

\subsection{Real-World Quanser Aero 2 System}

\paragraph{Mechatronic Background}

The Quanser Aero~2 is a motion control testbed for
multiple-input-multiple-output control theory with the goal to position the
beam by actuating two fans, see \cref{fig:aero2testbed}. It can be put into
a 1-DOF or 2-DOF configuration. In either case it displays pronounced
non-linear system behavior by design for educational reasons.

\begin{figure}[htb]
  \centering
  \includegraphics[width=0.35\textwidth]{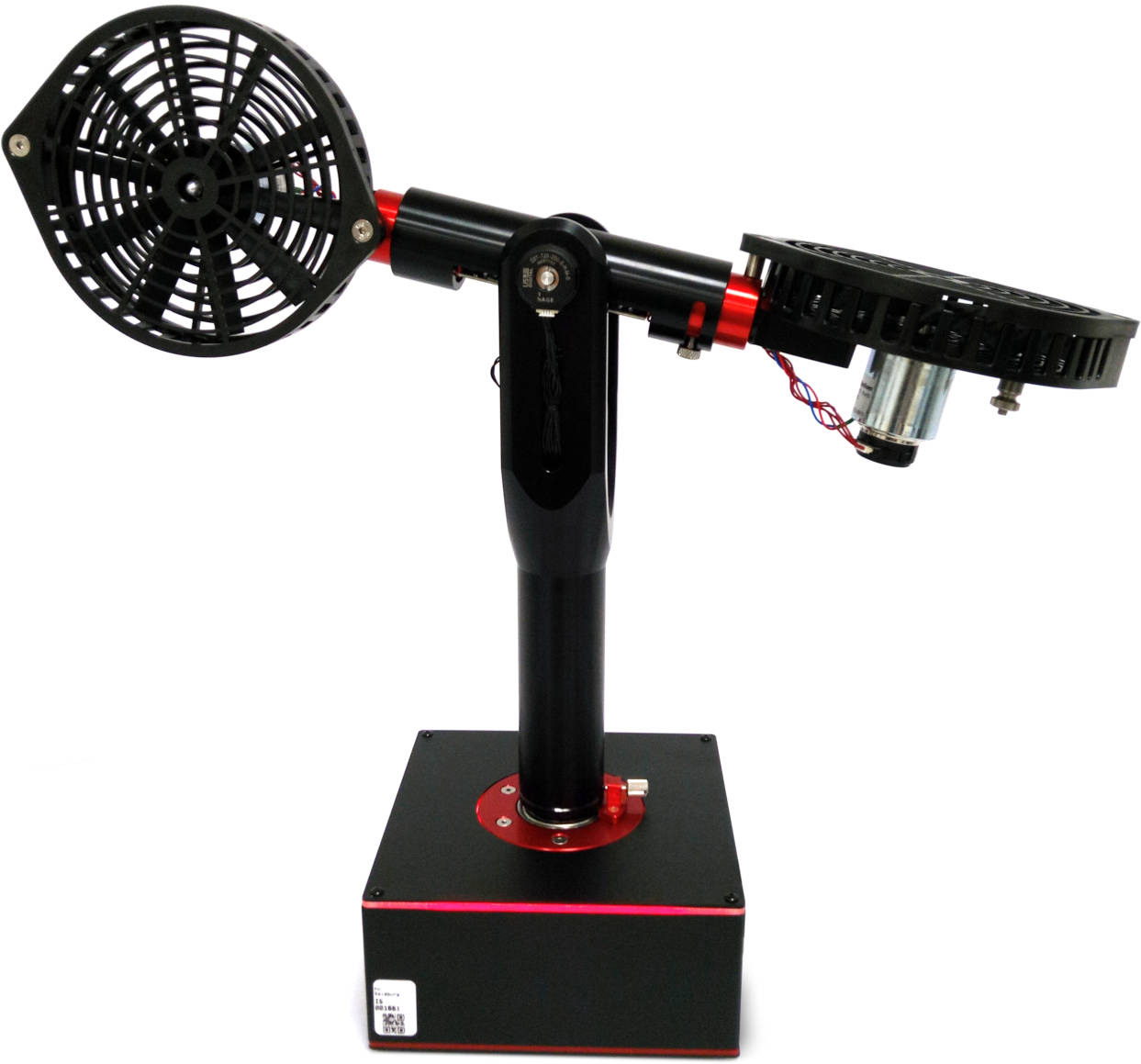}
  \caption{The Quanser Aero~2 testbed in 2-DOF configuration. We turn the fans
  in the same direction and lock the main arm in order to obtain the 1-DOF
configuration of the Aero~2.}
  \label{fig:aero2testbed}
\end{figure}

We follow the configuration in \cite{Schaefer2024a}, which is the 1-DOF
configuration. In this work, a control task is set up where the beam shall
follow a target pitch signal.
The \gls{rl} agent observes the actual pitch angle, actual angular velocity and
a target pitch angle, and outputs actions in form of a continuous voltage for
the fans.
The state at time $t$ is formalized as a vector $(\theta_t, \theta^*, \theta_t
- \theta_{t-1})$, where $\theta_t$ is the actual pitch angle, $\theta^*_t$ is
the target pitch angle, and $\theta_t - \theta_{t-1}$ captures the actual pitch
velocity, at time $t$, respectively. The action is a scalar output voltage.
The objective is to minimize the absolute tracking error over time, formalized
by the return $R = \sum_t r_t$, which is the cumulation of the reward $r_t
= -|\theta_t  - \theta^*_t|$.

\paragraph{Environment Implementation Details}

For the \gls{rl} training, we again leverage the Gymnasium environment framework
and use Stable Baselines3 implementations of \gls{ppo} and \gls{ars}. To this
end, we needed to implemented a custom Gymnasium environment that integrates the
Aero~2 system. This implementation uses a high-fidelity simulation model of the
Aero~2 system dynamics as introduced and verified by \cite{Schaefer2024b}. The code 
for this environment is published at \cite{UngerSchaefer2026}.

\paragraph{Training and Evaluation}

Since the control task is to follow a target pitch signal, for our training, we
generate a target pitch signal that (i) is continuous with limited dynamics
such that an agent can potentially follow given the dynamic ability of the
Aero~2 and (ii) also contains constant sections in order to judge the steady
state error, which is an interesting quality criterion in control theory. (Note
that return maximization implicitly leads to stead-state error minimization.)

To this end, we generate for each episode a target pitch signals that consist
of a random, continuous sequence of sinusoidal and constant sections. Its total
length is 2000 steps, which equals to \SI{200}{\second} by a sample time of
\SI{100}{\milli\second}. The sections are of uniform random length between
\SIrange{2.25}{15}{\second} and each section ends at a new intermediate uniform
random target pitch in the interval \SIrange{-40}{40}{\degree}, i.e.,
\SIrange{-0.698}{0.698}{\radian}. In \cref{fig:evaltrajectoriesaero} at the top
subfigure we see such a random target pitch signal.

\begin{figure}[h]
  \centering
  \includegraphics[height=85mm,trim={0mm 3.25mm 0mm 2.5mm},clip]{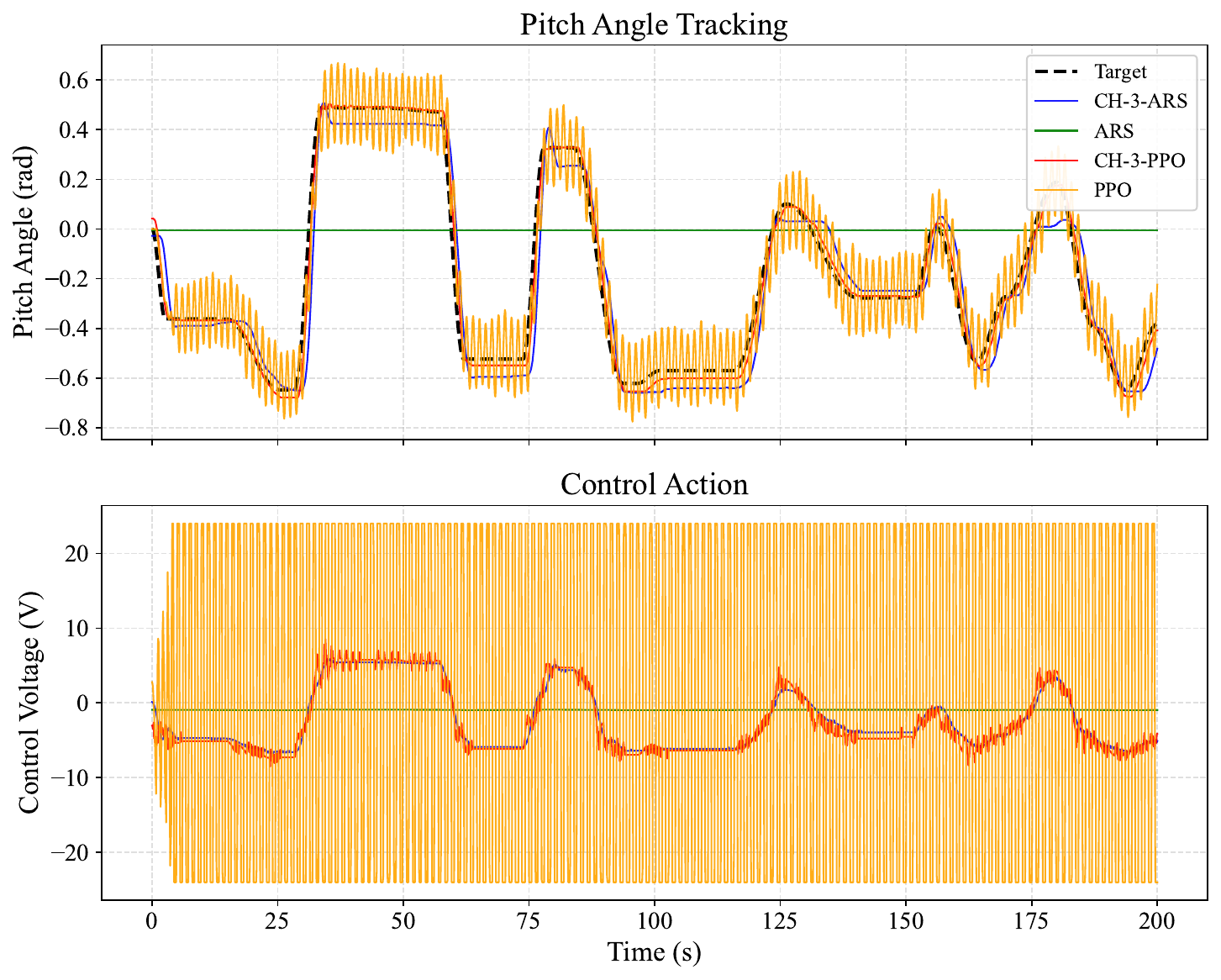}
  \caption{Details on one of ten evaluation trajectories performed on the real-world Quanser Aero 2 system.}
  \label{fig:evaltrajectoriesaero}
\end{figure}

We train \gls{mlp} and Chebyshev policies with \gls{ppo} and \gls{ars} as in the other
experiments. For the \gls{mlp} we again use the default 2-layer network
architecture of size $[64, 64]$, as suggested in \cite{Schaefer2024a}. For the
Chebyshev policies we found that max-degree 3 polynomials worked best. \Gls{ppo}
trained for \num{150000} steps and \gls{ars} trained for \num{4000000}, for both
\gls{mlp} and Chebyshev policies. As in previous experiments, we train 12
agents for each combination and select the best policy for evaluation.


We evaluate the policies on a set of 10 random target pitch signals and collect
the return, the mean power consumption and the mean action magnitude. The seed
of the evaluation set is fixed and therefore reproducible. The power
consumption is calculated by the product of motor coil current and voltage,
i.e., it is measured on the real system and simulated in Gymnasium. In
\cref{tab:evalpoliciesaero} we summarize the results when the evaluation is
performed in the simulation environment.

We furthermore evaluated the policies trained in simulation on the real-world
system, which provides us with insights on the sim-to-real transferability.
The results are summarized in \cref{tab:evalpoliciesaeroreal}.

The sim-to-real transferability is paramount for real-world \gls{rl}, because
from-scratch training of \gls{rl} agents on the real system is typically
prohibited due to safety limitations, risk of damage and wear, and is bound to
real-time speed and therefore virtually unfeasible. So even if we continue
training on the real-system, we typically have to start with agents pretrained
in simulation.

Let us first discuss the results from the simulation environment in
\cref{tab:evalpoliciesaero}. In all cases, Chebyshev policies again outperform
their \gls{mlp} counterparts.

In more detail, we recognize that \gls{ars} with an \gls{mlp} policy essentially
failed to control the beam in a way to follow the target pitch signal
sufficiently, leading to a high average pitch deviation (i.e., large negative
return), see also further details in \cref{fig:evaltrajectoriesaero}. (Note
that \cref{fig:evaltrajectoriesaero} actually displays the results on the real
system, but the main messages are the same for the simulation environment.) In
contrast, \gls{ars} did succeed for Chebyshev policies, although \textsc{CH-3-ARS}
clearly displays a significant steady-state error in
\cref{fig:evaltrajectoriesaero}.

\textsc{PPO} achieves a much better pitch deviation than \textsc{ARS}.
However, the output voltage (and power consumption) is remarkably high and
a closer look to \cref{fig:evaltrajectoriesaero} reveals the reason for this:
\textsc{PPO} displays pronounced oscillation around the target pitch caused by
the bang-bang strategy apparently learned by \gls{ppo}. However, again with Chebyshev
policies, this unfavorable behavior disappears and \textsc{CH-3-PPO} achieves
the best return and also displays most favorable dynamic behavior from
a control-theoretic point of view, which we see in power consumption and action
magnitude.

Finally, let us discuss the performance of the policies when evaluated on the
real system.
Since \textsc{ARS} failed to control the beam in simulation, it is expected
that it also does so on the real system and hence we receive essentially the
same pitch deviation. For \textsc{CH-3-ARS} we see that the pitch deviation
gets worse by \SI{31}{\percent}, so it moderately retained its performance.
We see that the transfer from simulation to the real system caused
\textsc{PPO} to increase its pitch deviation by a factor of \num{2.151}, while
\textsc{CH-3-PPO} essentially could retain its performance and only increased
the pitch deviation by \SI{13.4}{\percent}. As a result, on the real system
\textsc{CH-3-PPO} now outperforms \textsc{PPO} by a factor of 3.3 in terms of
pitch deviation.
On a side note, the bang-bang control of \textsc{PPO} actually caused
significant heat on the motors of the Aero~2, and required us to insert
cool-down phases in the evaluation procedure.

\begin{table}[h]
  \centering
  \caption{Performance evaluation on the Quanser Aero~2 Gymnasium simulation
  environment over $10$ evaluation episodes.
  For each episode we report on the average pitch deviation, average action
  magnitude and average power consumption and list it as $\mu \pm \sigma$,
  where $\mu$ is the mean and $\sigma$ is the standard deviation over the 10
  episodes. The mean return $\overline{R}$ over all episodes is the negative
  average pitch deviation multiplied by the number of episode steps, which is
  2000.
}
  \label{tab:evalpoliciesaero}
  \small
  \begin{tabular}{lcrrr}
    \toprule
    Pol.~$\pi$         & Pitch deviation (\unit{\radian}) $\downarrow$ & $\overline{R}$ $\uparrow$ & Action Magnitude (\unit{\volt}) $\downarrow$ & Power Consumption (\unit{\watt}) $\downarrow$  \\
    \midrule
    \textsc{CH-3-ARS}  & $0.0626 \pm 0.0113$                           &  $-125.2$                 & $4.0828  \pm 0.4391$                         & $1.3331  \pm 0.2465$                           \\
    \textsc{ARS}       & $0.3609 \pm 0.0615$                           &  $-721.8$                 & $0.9305  \pm 0.0063$                         & $0.0559  \pm 0.0008$                           \\ \midrule
    \textsc{CH-3-PPO}  & $0.0246 \pm 0.0038$                           &  $-49.2 $                 & $4.4338  \pm 0.5889$                         & $1.6545  \pm 0.3647$                           \\
    \textsc{PPO}       & $0.0423 \pm 0.0037$                           &  $-84.6 $                 & $16.3623 \pm 0.4699$                         & $37.0509 \pm 1.9717$                           \\
    \bottomrule
  \end{tabular}
\end{table}

\begin{table}[h]
  \centering
  \caption{Performance evaluation of agents trained in simulation on the real
    Quanser Aero 2 system without additional training in analogy to
    \cref{tab:evalpoliciesaero}.
    We also listed the quotient of the pitch deviation between the evaluation
    on the real and the simulated system.
  }
  \label{tab:evalpoliciesaeroreal}
  \small
  \begin{tabular}{lccrrr}
    \toprule
    Pol.~$\pi$         & Pitch deviation (\unit{\radian}) $\downarrow$& Real/sim deviation $\downarrow$ & $\overline{R}$ $\uparrow$ & Action Magnitude (\unit{\volt}) $\downarrow$ & Power Consumption (\unit{\watt}) $\downarrow$  \\
    \midrule
    \textsc{CH-3-ARS}  & $0.0821 \pm 0.0064$                          & $1.312$                         & $-164.2$                  & $4.1984 \pm 0.3757$                          & $0.3604  \pm 0.0765$             \\
    \textsc{ARS}       & $0.3592 \pm 0.0512$                          & $0.995$                         & $-718.4$                  & $0.9275 \pm 0.0067$                          & $0.0049  \pm 0.0006$             \\ \midrule
    \textsc{CH-3-PPO}  & $0.0279 \pm 0.0054$                          & $1.134$                         & $-55.8 $                  & $4.6035 \pm 0.3789$                          & $0.4966  \pm 0.0857$             \\
    \textsc{PPO}       & $0.0910 \pm 0.0064$                          & $2.151$                         & $-182.0$                  & $21.1870 \pm 0.9451$                         & $52.6000 \pm 2.2246$             \\
    \bottomrule
  \end{tabular}
\end{table}



\end{document}